\documentclass{article}

\PassOptionsToPackage{numbers, sort&compress}{natbib}
\usepackage[preprint]{neurips_2026}

\usepackage[utf8]{inputenc} 
\usepackage[T1]{fontenc}    
\usepackage{hyperref}       
\usepackage{url}            
\usepackage{booktabs}       
\usepackage{amsfonts}       
\usepackage{nicefrac}       
\usepackage{microtype}      
\usepackage{xcolor}         
\usepackage{enumitem}
\usepackage{amsmath}
\usepackage{amssymb}
\usepackage{amsthm}
\usepackage{wrapfig}
\usepackage{graphicx}
\usepackage{algorithm}
\usepackage{algorithmic}
\usepackage{subcaption}
\usepackage{multirow}
\usepackage{wrapfig}

\newtheorem{lemma}{Lemma}
\title{GCCM: Enhancing Generative Graph Prediction via Contrastive Consistency Model}

\author{%
}

\author{%
  Shaozhen Ma \\
  University of New South Wales \\
  \texttt{shaozhen.ma@unsw.edu.au} \\
  \And
  Wei Huang \\
  University of New South Wales \\
  \texttt{w.c.huang@unsw.edu.au} \\
  \And
  Hanchen Wang\\
  University of Technology Sydney \\
  \texttt{Hanchen.Wang@uts.edu.au} \\
  \And
  Dong Wen\\
  University of New South Wales\\
  \texttt{dong.wen@unsw.edu.au} \\
  \And
  Wenjie Zhang\\
  University of New South Wales\\
  \texttt{wenjie.zhang@unsw.edu.au} \\
}

\begin{document}

\maketitle

\begin{abstract}
Conditional generative models, particularly diffusion-based methods, have recently been applied to graph prediction by modeling the target as a conditional distribution given the input graph, yielding competitive results compared to deterministic predictor.
However, existing diffusion-based prediction methods typically require expensive iterative denoising at inference and often suffer from unstable sampling, which motivates recent efforts to reduce inference denoising steps and enable stable sampling via techniques such as consistency training. 
Despite this progress, we find that existing consistency training methods for graph prediction could potentially fall into a shortcut solution: the model may attempt to satisfy the self-consistency constraint by ignoring the noisy target (i.e., assigning it negligible weight), ultimately collapsing into a purely deterministic predictor. 
To mitigate such shortcut solution, we propose GCCM, a graph contrastive consistency model that goes beyond isolated pairwise matching between the same target at different noise levels by introducing negative pairs into a contrastive consistency objective.
This adds an additional separation requirement, making the shortcut solution no longer trivially sufficient to satisfy the proposed objective.
Moreover, we apply feature perturbation to the input node/edge features to break identical conditioning on the input graph, so that the shortcut no longer yields the same predictions across noise levels and becomes less attractive.
Extensive experiments on benchmark datasets demonstrate that GCCM mitigates the shortcut solution and yields consistent performance improvements in graph prediction compared to deterministic predictors.
\end{abstract}

\section{Introduction}
Prediction tasks on graphs (e.g., node/graph-level predictions) are fundamental problems in graph machine learning, aiming to predict target variables defined over graph-structured data.
Earlier methods for graph prediction~\cite{kreuzer2021rethinkinggraphtransformersspectral,ying2021transformersreallyperformbad,rampášek2023recipegeneralpowerfulscalable} mainly focus on a deterministic paradigm, which directly maps an input node or graph to the target.
In recent years, generative models, particularly the diffusion-based models, have demonstrated promising performance across a wide range of generation tasks on graphs~\cite{sun2023difuscographbaseddiffusionsolvers,vignac2023digressdiscretedenoisingdiffusion,huang2025diffgedcomputinggraphedit,huang2025unsupervisedtrainingmatchingbasedgraph}.
Benefiting from their strong capability in modeling complex data distributions, increasing attention has been paid to extending diffusion-based paradigms to graph prediction tasks.
In contrast to deterministic models, diffusion-based approaches reformulate prediction as a conditional generation problem, modeling the full conditional distribution of target variables rather than producing a single point estimate.
Despite their potential, applying diffusion-based models to graph prediction tasks remains challenging in two key aspects: (1) \textbf{Step-Level Inefficiency:} diffusion-based methods perform inference through iterative denoising, frequently involving dozens or even hundreds of reverse steps, which introduces substantial computational overhead, limiting their inference efficiency; (2) \textbf{Sample-Level Inefficiency:} predictions obtained from a single noise sample can exhibit high variance, leading to unstable inference. To obtain reliable predictions, generative methods typically require aggregating predictions from multiple noise samples, which further increases inference cost.

To address the above challenges, a common strategy is to adopt the consistency training~\cite{song2023consistencymodels,song2024improved} to enhance diffusion-based models. 
Specifically, PCL~\cite{li2025generative} proposes a predictive consistency framework to train the diffusion-based graph prediction model by enforcing consistency mapping across different noise levels, which enables stable sampling and one-step inference from a single noise sample.
Despite its efficiency, we observe that applying consistency training to graph prediction admits a shortcut solution.
When predictions at different noise levels are conditioned on the same underlying input graph and enforced to be matched, the model can trivially satisfy the consistency objective by ignoring the noisy target and relying solely on the conditioning input.
This potential shortcut collapses the conditional denoising network into a purely deterministic predictor, thus undermining the goal of conditional generation.

\textbf{Our Contributions.}\quad Motivated by the above observation, we propose Graph Contrastive Consistency Models (GCCM) to mitigate such shortcut solution in consistency-based prediction.
In particular, \textbf{(1)} instead of treating self-consistency as a purely matching constraint between two denoising predictions of the same underlying instance (i.e., a node of node-level prediction and a graph for graph-level prediction), we incorporate negative pairs across different instances in the mini-batch.
This yields a contrastive consistency objective that aligns the target representations of the positive pair while separating them from negatives in the latent space, making it harder to satisfy the contrastive consistency objective by simply ignoring the noisy target.
Nevertheless, the proposed contrastive consistency objective alone does not fully eliminate the shortcut solution.
When the conditioning node/edge features remains identical across the two noisy views, the model may still attempt to satisfy the positive-pair agreement by relying solely on the identical conditioning input and ignoring the noisy target (i.e., the shortcut solution).
\textbf{(2)} To further enhance the proposed contrastive consistency objective, we apply feature perturbation to the input node/edge features to break identical conditioning, making the shortcut solution less attractive.
Together, these two components encourage our proposed GCCM to fully leverage the generative formulation to enhance the performance of predictors.
Experimental results show that GCCM outperforms the existing diffusion-based graph prediction baselines on wide-range benchmark datasets and prediction tasks. Detailed ablation studies further verify the contribution of each proposed component in mitigating the shortcut solution.

\vspace{-0.3cm}

\section{Related Works}

\noindent\textbf{Generative diffusion models for graph prediction.}\quad
In recent years, diffusion models~\cite{ho2020denoisingdiffusionprobabilisticmodels,song2022denoisingdiffusionimplicitmodels,nichol2021improveddenoisingdiffusionprobabilistic,rombach2022highresolutionimagesynthesislatent} have emerged as a powerful framework for a wide range of generation tasks. 
Motivated by their strong generative capacity, a growing body of work has extended diffusion models to graph-related generation tasks~\cite{vignac2023digressdiscretedenoisingdiffusion,zhou2024unifyinggenerationpredictiongraphs}. 
Beyond graph generation tasks, LGD~\cite{zhou2024unifyinggenerationpredictiongraphs} explores diffusion models for graph prediction by applying diffusion in a latent space to generate latent representations and then mapping them to the target space using a learned decoder.
Although effective, it still incurs substantial computational overhead due to iterative denoising at inference time.
Building upon diffusion models, consistency models~\cite{song2023consistencymodels,song2024improved} eliminate iterative reverse diffusion by directly learning a mapping from noisy samples to clean data through the self-consistency objective, enabling efficient one-step or few-step generation.
Inspired by the consistency mapping principle, PCL~\cite{li2025generative} further develops a predictive consistency framework for prediction tasks by enforcing consistent predictions across different noise levels.
This formulation enables stable sampling and one-step inference in diffusion-based graph prediction models.
Due to the space limitations, an extended related work regarding the classical deterministic graph prediction models can be found in Appendix~\ref{appendix:related_works}.
\vspace{-0.3cm}

\section{Preliminaries}
\label{sec:preliminaries}
In this section, we use $\mathbf{Y}$ to denote the target variable to be generated, and $\mathbf{C}$ to denote the conditioning input for conditional generation. 

\subsection{Diffusion Models}
\label{sub:Diffusion_Models}
Diffusion models~\cite{sohldickstein2015deepunsupervisedlearningusing,austin2023structureddenoisingdiffusionmodels,hoogeboom2021argmaxflowsmultinomialdiffusion,ho2020denoisingdiffusionprobabilisticmodels,nichol2021improveddenoisingdiffusionprobabilistic} define a generative process by gradually transforming the clean target $\mathbf{Y}_0$ into noise through a forward process, and learning to recover the clean target progressively through a reverse process.

\subsubsection{Forward Process}
\label{sub:diffusion_forward_process}
The forward process gradually corrupts the clean target $\mathbf{Y}_0$ into a sequence of noisy latent variables
$\mathbf{Y}_{1:T}=\{\mathbf{Y}_1,\mathbf{Y}_2,\ldots,\mathbf{Y}_T\}$, defined as
$q(\mathbf{Y}_{1:T}\mid \mathbf{Y}_0)=\prod_{t=1}^{T} q(\mathbf{Y}_t\mid \mathbf{Y}_{t-1})$.

\noindent\textbf{Discrete targets.}\quad
For discrete targets, the clean target is represented in one-hot form as
$\mathbf{Y}_0=\mathbf{Y}\in\{0,1\}^{n\times K}$,
where $n$ is the number of target variables and $K$ is the number of categories.
The forward process is defined as
$q(\mathbf{Y}_t\mid \mathbf{Y}_{t-1})
=
\mathrm{Cat}\left(\mathbf{Y}_t;\; \mathbf{p}=\mathbf{Y}_{t-1}\mathbf{Q}_t\right)$,
where $\mathrm{Cat}(\mathbf{Y}_t;\mathbf{p})$ denotes a categorical distribution over $K$ one-hot vectors with probabilities given by $\mathbf{p}$, and
$
\mathbf{Q}_t
=
(1-\beta_t)\mathbf{I}
+
\frac{\beta_t}{K}\mathbf{1}\mathbf{1}^{\top}
\in \mathbb{R}^{K\times K},
$
is the transition matrix that determines the corruption introduced at time step $t$, where $\beta_t\in(0,1)$ is the corruption rate.


\noindent\textbf{Continuous targets.}\quad Let $\mathbf{Y}_0=\mathbf{Y}\in\mathbb{R}^{n\times 1}$ denote the clean target.
The forward process $q(\mathbf{Y}_t\mid \mathbf{Y}_{t-1})$ for the continuous targets perturbs $\mathbf{Y}_0$ by progressively injecting standard Gaussian noise
$\boldsymbol{\epsilon}\sim\mathcal{N}(\mathbf{0},\mathbf{I})$ over time steps, such that $
q(\mathbf{Y}_t\mid \mathbf{Y}_{t-1})
=
\mathcal{N}\left(
\mathbf{Y}_t;\;
\sqrt{1-\beta_t}\,\mathbf{Y}_{t-1},\;
\beta_t\mathbf{I}
\right)
$,
where $\beta_t\in(0,1)$ controls the noise level at time step $t$.

\subsubsection{Reverse Process}
The reverse process reconstructs the clean target by defining a sequence of learned conditional transitions
$p_{\theta}(\mathbf{Y}_{0:T}\mid \mathbf{C})
= p(\mathbf{Y}_T)\,\prod_{t=1}^{T} p_{\theta}\left(\mathbf{Y}_{t-1}\mid \mathbf{Y}_t,\mathbf{C}\right)$ that progressively transforms $\mathbf{Y}_T$ toward $\mathbf{Y}_0$.
Each transition $p_{\theta}\left(\mathbf{Y}_{t-1}\mid \mathbf{Y}_t,\mathbf{C}\right)$ is parameterized by the denoising network $f_{\theta}$.
During training, diffusion models sample a time step $t$ and obtain the corresponding noisy target $\mathbf{Y}_t$ via the forward process.
During inference, they start from a random noisy target $\mathbf{Y}_T$ and iteratively apply the learned reverse transitions to denoise $\mathbf{Y}_T$ towards the target distribution.

\noindent\textbf{Discrete targets.}\quad For discrete targets, a common parameterization is to train a denoising network $f_{\theta}(\mathbf{Y}_t,t,\mathbf{C})$ to predict the clean-target distribution, and then obtain the reverse distribution by marginalizing over the predicted clean target $\tilde{\mathbf{Y}}_0$, such that $p_{\theta}(\mathbf{Y}_{t-1}\mid \mathbf{Y}_t,\mathbf{C})
=
\sum_{\tilde{\mathbf{Y}}}
q(\mathbf{Y}_{t-1}\mid \mathbf{Y}_t,\tilde{\mathbf{Y}}_0)\;
f_{\theta}(\mathbf{Y}_t,t,\mathbf{C})$, where the posterior $q(\mathbf{Y}_{t-1}\mid \mathbf{Y}_t,\mathbf{Y}_0)$ can be obtained by $q(\mathbf{Y}_{t-1}\mid \mathbf{Y}_t,\mathbf{Y}_0)
=
\frac{
q(\mathbf{Y}_t\mid \mathbf{Y}_{t-1},\mathbf{Y}_0)\;
q(\mathbf{Y}_{t-1}\mid \mathbf{Y}_0)
}{
q(\mathbf{Y}_t\mid \mathbf{Y}_0)
}$.
The training objective minimizes the expected discrepancy between the predicted target and the ground-truth clean target: $\mathcal{L}_{\mathrm{disc}}(\theta)
=
\mathbb{E}
\Big[
d\big(
f_{\theta}(\mathbf{Y}_t,t,\mathbf{C}),\,
\mathbf{Y}_0
\big)
\Big]$, where $d(\cdot,\cdot)$ is typically the cross-entropy loss for discrete targets.

\noindent\textbf{Continuous targets.}\quad 
For continuous targets, the reverse Gaussian transition is defined as $p_\theta(\mathbf{Y}_{t-1}\mid \mathbf{Y}_t,\mathbf{C})
=
\mathcal{N}\Big(
\mathbf{Y}_{t-1};
\mu_\theta(\mathbf{Y}_t,t,\mathbf{C}),
\beta_t\mathbf{I}
\Big)$,
where $\mu_\theta(\mathbf{Y}_t,t,\mathbf{C})$ and $\beta_t\mathbf{I}$ denote the mean and covariance of the reverse transition, respectively.
Following the noise-prediction parameterization, the denoising network predicts the injected Gaussian noise $\hat{\boldsymbol{\epsilon}} = (\mathbf{Y}_t - \sqrt{\bar{\alpha}_t}\,\mathbf{Y}_0)\,/\,\sqrt{1-\bar{\alpha}_t} = f_\theta(\mathbf{Y}_t,t,\mathbf{C})$, which induces the mean function $\mu_\theta(\mathbf{Y}_t,t,\mathbf{C})
=
\frac{1}{\sqrt{\alpha_t}}
\Big(
\mathbf{Y}_t-\frac{\beta_t}{\sqrt{1-\bar{\alpha}_t}}\,f_{\theta}(\mathbf{Y}_t,t,\mathbf{C})
\Big)$.
The training objective minimizes the expected discrepancy between the predicted noise and the injected noise: $\mathcal{L}_{\mathrm{cont}}(\theta)
=
\mathbb{E}
\Big[
d\big(
f_{\theta}(\mathbf{Y}_t,t,\mathbf{C}),\,
\boldsymbol{\epsilon}
\big)
\Big]$, where $d(\cdot,\cdot)$ is typically MSE for continuous targets.

\subsection{Consistency Models}
\label{sub:consistency_models}
Consistency models~\cite{song2023consistencymodels,song2024improved} are a family of generative models that aim to enable stable sampling and one-step inference for the diffusion reverse process by directly predicting the clean target from a randomly sampled noisy target, such that
$
\hat{\mathbf{Y}}_0^{T} = f_{\theta}(\mathbf{Y}_T, T, \mathbf{C}).
$
To achieve this, unlike diffusion models that sample only one noisy target in the forward process, consistency models sample two noisy targets $\mathbf{Y}_{t_1}$ and $\mathbf{Y}_{t_2}$ by sampling two time steps $t_1$ and $t_2$.
The denoising network $f_\theta$ is then required to satisfy two key properties:
\textbf{(i) Boundary condition.}
At noise level $t=0$, the output should match the clean target:
$
f_{\theta}(\mathbf{Y}_{0}, 0, \mathbf{C})
=
\mathbf{Y}_{0}
$;
\textbf{(ii) Self-consistency.}
Outputs across different noise levels should be consistent:
$
f_{\theta}(\mathbf{Y}_{t_1}, t_1, \mathbf{C})
=
f_{\theta}(\mathbf{Y}_{t_2}, t_2, \mathbf{C}).
$

To satisfy the boundary condition, consistency models adopt a skip-connection parameterization:
\begin{equation}\small
f_{\theta}(\mathbf{Y}_t,t,\mathbf{C})
=
c_{\mathrm{skip}}(t)\,\mathbf{Y}_0
+
c_{\mathrm{out}}(t)\,F_{\theta}(\mathbf{Y}_t,t,\mathbf{C}),
\label{eq:cm_skip_connection}
\end{equation}
where $F_{\theta}$ is a free-form neural network, and $c_{\mathrm{skip}}(t)$ and $c_{\mathrm{out}}(t)$ are differentiable functions such that $c_{\mathrm{skip}}(0)=1$ and $c_{\mathrm{out}}(0)=0$.
And to satisfy the self-consistency property, the training objective of consistency models is defined as
\begin{equation}\small
\begin{aligned}
\mathcal{L}_{\mathrm{CM}}(\theta)
&=
\mathbb{E}\Big[
d\Big(
f_{\theta}(\mathbf{Y}_{t_1},t_1,\mathbf{C}),\;
f_{\theta^-}(\mathbf{Y}_{t_2},t_2,\mathbf{C})
\Big)
\Big],
\end{aligned}
\label{eq:cm_objective_graph}
\end{equation}
where $\theta^-$ denotes $\mathrm{stopgrad}(\theta)$.

\section{Methodology}
\label{sec:methodology}

\subsection{Graph Prediction as Conditional Generation \& Identification of the Shortcut Solution}
\label{sub:prediction_as_generation}


In this section, we formulate each prediction task on graphs using node embeddings for simplicity.
Consider a graph $\mathcal{G}=(\mathbf{X},\mathbf{A})$ with a set of $n$ nodes, where $\mathbf{X}\in\mathbb{R}^{n\times d_x}$ is the node feature matrix, $\mathbf{A}\in\{0,1\}^{n\times n}$ is the adjacency matrix.
We denote $\mathbf{Y}\in\mathbb{R}^{n\times d_k}$ as the target node labels for node-level prediction tasks, and $\mathbf{Y}\in\mathbb{R}^{1\times d_k}$ as the target graph label for graph-level prediction tasks. Specifically, $d_k$ denotes the label dimension, where $d_k=K$ for classification tasks with one-hot labels, and $d_k=1$ for regression tasks.

\begin{figure*}[t]
  \centering
  \includegraphics[width=\textwidth]{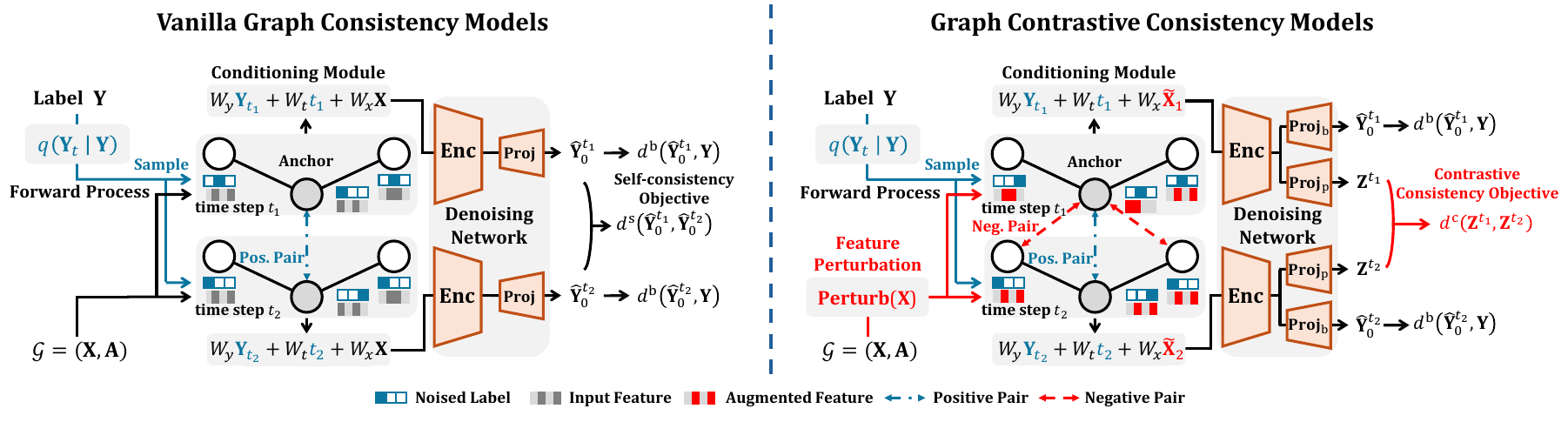}
  \vspace{-0.5cm}
    \caption{Illustration of Graph Contrastive Consistency Models (GCCM).
    Compared with vanilla consistency training for graph prediction, which uses $d^b$ to enforce the boundary condition and $d^s$ to enforce the self-consistency constraint, GCCM introduces contrastive consistency via $d^c$ on representations from two independently noisy labels, and further mitigates the shortcut solution via feature perturbation that breaks identical conditioning across consistency passes.}
  \label{fig_contrastive_framework}
  \vspace{-0.5cm}
\end{figure*}

To formulate prediction tasks on graphs as conditional generation, we adopt a diffusion model and treat the input graph $\mathcal{G}=(\mathbf{X},\mathbf{A})$ as the conditioning signal and the target node/graph labels $\mathbf{Y}$ as the generation target. 
We then apply the forward process as described in Section~\ref{sec:preliminaries} to obtain noisy labels $\mathbf{Y}_{t}$ at noise level $t$.
Notably, for graph-level prediction, we apply the forward process to the graph label $\mathbf{Y}\in\mathbb{R}^{1\times d_k}$ to obtain the noisy graph label $\mathbf{Y}_{t}\in\mathbb{R}^{1\times d_k}$, and then broadcast it to all nodes to form $\mathbf{Y}_{t}\in\mathbb{R}^{n\times d_k}$.
Next, the denoising network $f_\theta$ takes $(\mathbf{Y}_{t},t, \mathcal{G})$ as input and first uses a conditioning module~\cite{zhou2024unifyinggenerationpredictiongraphs,li2025generative} to fuse $(\mathbf{Y}_{t},t, \mathbf{X})$, such that:
\begin{equation}\small
\mathbf{H}_t^{(0)} = 
W_y\,\mathbf{Y}_t
+ W_t\,\mathbf{t}_{\mathrm{emb}}(t)
+ W_x\,\mathbf{X},
\label{eq:additive_fusion}
\end{equation}
Here, $\mathbf{H}^{(0)}\in\mathbb{R}^{n\times d_h}$ denotes the node-level fused initial hidden embedding matrix,
$\mathbf{t}_{\mathrm{emb}}(t)$ is the timestep embedding,
$W_y$, $W_t$, and $W_x$ are learnable projection matrices that map $\mathbf{Y}_{t}$, $\mathbf{t}_{\mathrm{emb}}(t)$, and $\mathbf{X}$ into the same $d_h$-dimensional space.
A graph encoder $\mathrm{Enc}(\cdot)$ is then applied to take $(\mathbf{H}_t^{(0)}, \mathbf{A})$ as inputs and update the node embeddings to $\mathbf{H}_t$. Finally, a predictive projection head $\mathrm{Proj}_{p}(\cdot)$ is applied to $\mathbf{H}_t$ to estimate the clean node/graph labels $\hat{\mathbf{Y}}_0^{t}$:
\begin{equation}\small
\begin{aligned}
\hat{\mathbf{Y}}_0^{t}
&= f_\theta\big(\mathbf{Y}_t, t, \mathcal{G}\big)
 = f_\theta\big(\mathbf{H}^{(0)}_t,\mathbf{A}\big) 
= \mathrm{Proj}_{p}\left(\mathrm{Enc}\left(\mathbf{H}^{(0)}_t, \mathbf{A}\right)\right)
\end{aligned}
\label{eq:pred_mapping}
\end{equation}
To accelerate diffusion models, a common strategy is to adopt consistency models for stable sampling and one-step inference.
However, the consistency models in Section~\ref{sub:consistency_models} are not directly tailored to prediction settings, mainly because their skip-connection parameterization in Eq.~\ref{eq:cm_skip_connection} assumes that the clean labels $\mathbf{Y}_0$ and the model output lie in the same continuous space and can be combined additively.
For tasks like graph classification, the model outputs a categorical distribution, whereas the clean labels are discrete class labels.
As a result, directly enforcing the boundary condition through additive mixing is less natural and usually requires task-specific adaptations.
To address this issue, PCL~\cite{li2025generative} enhances consistency models by enforcing the boundary condition through an explicit supervised objective, rather than an additive skip-connection parameterization, which allows the model to accommodate different types of labels.
Specifically, unlike the vanilla consistency objective in Eq.~\ref{eq:cm_objective_graph}, PCL adopts a triadic distance formulation to jointly enforce these two properties, yielding the following training objective:
\begin{equation}\small
\begin{aligned}
\mathcal{L}(\theta)
 = \mathbb{E}\Big[
\lambda_1 \Big(
d^b\left(
\hat{\mathbf{Y}}_0^{t_1}, \mathbf{Y}
\right)
+
d^b\left(
\hat{\mathbf{Y}}_0^{t_2}, \mathbf{Y}
\right)
\Big) 
+ \lambda_2\, d^s\left(
\hat{\mathbf{Y}}_0^{t_1},
\hat{\mathbf{Y}}_0^{t_2}
\right)
\Big].
\end{aligned}
\label{equ_loss_function}
\end{equation}
Here, the boundary condition objective $d^b(\cdot,\cdot)$ is instantiated as a task-specific supervised loss, e.g., cross-entropy for classification tasks and mean absolute error for regression tasks on graphs, while the self-consistency objective $d^s(\cdot,\cdot)$ is typically implemented using the MSE between the two predictions.
For notational convenience, throughout this paper we instantiate $d^s$ as the MSE distance unless otherwise specified.
The overall framework of PCL for graph prediction is illustrated in Fig.~\ref{fig_contrastive_framework}.


While the above consistency training framework is simple and effective, we find that minimizing the training objective in Eq.~\eqref{equ_loss_function} can lead to a shortcut solution in graph prediction settings.
The underlying reason is that the self-consistency objective $d^s$ enforces agreement between two predictions $\hat{\mathbf{Y}}_0^{t_1} := f_{\theta}(\mathbf{Y}_{t_1}, t_1, \mathcal{G})$ and $\hat{\mathbf{Y}}_0^{t_2} := f_{\theta}(\mathbf{Y}_{t_2}, t_2, \mathcal{G})$, both of which are conditioned on the same input graph $\mathcal{G}$. 
Under this setup, the only factors that vary across the two passes are the noisy labels $\mathbf{Y}_t$ and the corresponding time step $t$. 
Consequently, reducing the influence of $(\mathbf{Y}_t, t)$ on the predictions becomes an easy way to satisfy the agreement requirement, which can be achieved by down-weighting the corresponding fusion terms via $W_y$ and $W_t$ in Eq.~\eqref{eq:additive_fusion}. 
Moreover, the boundary condition objective $d^b$ does not prevent the model from converging to such a shortcut solution, as we will formalize in the following lemma:

\begin{lemma}
\label{lem:mse_shortcut_family}
\textcolor{blue}{}
Consider the training objective in Eq.~\eqref{equ_loss_function} $\mathcal{L}(\theta) = \mathbb{E}\left[\lambda_1\left(d^b(\hat{\mathbf{Y}}_0^{t_1}, \mathbf{Y}) + d^b(\hat{\mathbf{Y}}_0^{t_2}, \mathbf{Y})\right) + \lambda_2\, d^s(\hat{\mathbf{Y}}_0^{t_1}, \hat{\mathbf{Y}}_0^{t_2})\right]$,
where each predicted label $\hat{\mathbf{Y}}_0^{t}$ can be rewritten as $\hat{\mathbf{Y}}_0^{t}
= f_\theta\big(\mathbf{Y}_t, t, \mathcal{G}\big)
= f_\theta\big(\mathbf{H}^{(0)}_t, \mathbf{A}\big) 
= f_\theta\Big(
W_y\,\mathbf{Y}_t
+ W_t\,\mathbf{t}_{\mathrm{emb}}(t)
+ W_x\,\mathbf{X},
\mathbf{A}
\Big)$.
Minimizing the above objective (Eq.~\eqref{equ_loss_function}) could potentially lead the denoising network $f_\theta$ to a shortcut solution set, such that the noisy labels $\mathbf{Y}_t$ and timestep $t$ are ignored:
\begin{equation}\small
\mathcal{F}_{\mathrm{sc}} := \{\theta:\; W_y=\mathbf{0}_{d_k\times d_h},\, W_t=\mathbf{0}_{d_t\times d_h}\},
\label{eq:shortcut_family}
\end{equation}
where $d_k$ denotes the dimension of the input label.
\end{lemma}

A formal proof is provided in Appendix~\ref{sec:proof_mse_shortcut}.
Under the shortcut solution set $\mathcal{F}_{\mathrm{sc}}$ characterized in Lemma~\ref{lem:mse_shortcut_family}, Eq.~\eqref{eq:additive_fusion} simplifies to
$
\mathbf{H}^{(0)}_t = W_x\,\mathbf{X},
$
and the denoising network collapses into a deterministic predictor $g_\theta(\cdot)$ that depends only on the input graph $\mathcal{G}$, i.e.,
$
f_\theta(\mathbf{Y}_t, t, \mathcal{G}) \;:=\; g_\theta(\mathcal{G}).
$
Since both predictions become identical regardless of the noise level, the self-consistency objective attains its minimum value $0$, and the training objective in Eq.~\eqref{equ_loss_function} can be simplified as:
$\mathcal{L}(\theta)
=
\mathbb{E}\Big[
\lambda_1 \Big(
d^b\left(
g_\theta(\mathcal{G}), \mathbf{Y}
\right)
+
d^b\left(
g_\theta(\mathcal{G}), \mathbf{Y}
\right)
\Big)
+ \lambda_2 \cdot 0
\Big] 
=
2\lambda_1\,\mathbb{E}\Big[
d^b\left(
g_\theta(\mathcal{G}), \mathbf{Y}
\right)
\Big]
+ \lambda_2 \cdot 0$.
The resulting objective is equivalent to a standard deterministic supervised loss, which the model can minimize in the same way as a deterministic predictor to reduce the overall loss. 
Therefore, the boundary condition objective $d^b$ does not impose any additional restriction that prevents the model from taking the shortcut solution.

\subsection{Mitigating the Shortcut Solution via the Contrastive Consistency Objective}
\label{sub:contrastive_consistency}

The reason why minimizing the training objective in Eq.~\eqref{equ_loss_function} could lead to the shortcut solution in Lemma~\ref{lem:mse_shortcut_family} is that the self-consistency objective $d^s$ only enforces pairwise matching between two predictions from the same underlying instance (i.e., a node for node-level prediction and a graph for graph-level prediction), and thus can be trivially satisfied by ignoring the noisy labels and the corresponding time step.
Since the boundary condition objective $d^b$ does not prevent this shortcut, mitigating it requires modifying the self-consistency objective itself.
To mitigate this shortcut, we incorporate negative pairs across different instances and propose a graph contrastive consistency objective to replace the self-consistency objective.
This introduces an extra separation requirement beyond positive pairwise matching, making the shortcut solution in Eq.~\eqref{eq:shortcut_family} no longer trivially sufficient to satisfy the self-consistency constraint.

Specifically, given a mini-batch of $N$ instances with clean labels $\{\mathbf{y}_{0,i}\}_{i=1}^{N}$,
we apply the forward process at two time steps $t_1$ and $t_2$ to obtain noisy label sets
$\{\mathbf{y}_{t_1,i}\}_{i=1}^{N}$ and $\{\mathbf{y}_{t_2,i}\}_{i=1}^{N}$.
Taking $\mathbf{y}_{t_1,i}$ as the anchor, we treat $(\mathbf{y}_{t_1,i}, \mathbf{y}_{t_2,i})$ as a positive pair and use the remaining noisy labels in the second set as negatives, i.e., $\{\mathbf{y}_{t_2,j}\}_{j\neq i}$.
In this setting, we compute the contrastive consistency objective on the corresponding predicted clean labels
$\{\hat{\mathbf{y}}_{0,i}^{t_1}\}_{i=1}^{N}$ and $\{\hat{\mathbf{y}}_{0,i}^{t_2}\}_{i=1}^{N}$ in label space, to align the positive pairs while separating the negatives.
However, directly computing the contrastive consistency objective in the label space may induce invariance and information compression.
Since the prediction outputs in this space are primarily optimized to satisfy the boundary condition objective $d^b(\cdot,\cdot)$, such effects can undesirably reshape the predictions and interfere with the supervision imposed by $d^b(\cdot,\cdot)$.

To alleviate this issue, we compute the contrastive consistency objective in a latent space, allowing alignment and separation of latent representations induced by instances, while avoiding direct constraints on the prediction outputs that are required to satisfy the boundary condition objective.
Specifically, we apply an additional projection head $\mathrm{Proj}_{c}(\cdot)$ to the shared encoder output $\mathbf{h}_t$ to obtain a latent representation $\mathbf{z}_i^{t}\in\mathbb{R}^{d_z}$ for each noisy label $\mathbf{y}_{t,i}$ in the mini-batch.
The contrastive consistency objective $d^c(\cdot,\cdot)$ is then computed on these latent representations, while the boundary condition objective $d^b(\cdot,\cdot)$ is kept unchanged in the label space.
With the latent representations $\{\mathbf{z}_i^{t_1}\}_{i=1}^{N}$ and $\{\mathbf{z}_i^{t_2}\}_{i=1}^{N}$, we define the contrastive consistency loss from $t_1$ to $t_2$ as
\begin{equation}\small
\begin{aligned}
\mathcal{L}_{\mathrm{CL}}^{t_1\to t_2}
=
-\frac{1}{N}
\sum_{i=1}^{N}
\log
\frac{
\exp\left(\operatorname{sim}\left(\mathbf{z}_i^{t_1},\mathbf{z}_i^{t_2}\right)/\tau\right)
}{
\sum_{j=1}^{N}
\exp\left(\operatorname{sim}\left(\mathbf{z}_i^{t_1},\mathbf{z}_j^{t_2}\right)/\tau\right)
}, 
\text{with}\;
\operatorname{sim}(\mathbf{z}_{i}^{t_1},\mathbf{z}_{j}^{t_2})
=
\frac{(\mathbf{z}_{i}^{t_1})^{T}\mathbf{z}_{j}^{t_2}}
{\|\mathbf{z}_{i}^{t_1}\|\,\|\mathbf{z}_{j}^{t_2}\|}.
\end{aligned}
\label{equ_latent_contrastive_obj}
\end{equation}
where $\operatorname{sim}(\cdot,\cdot)$ denotes cosine similarity, and $\tau$ is a temperature parameter.
The reverse direction from $t_2$ to $t_1$ is defined analogously.
For notational convenience, we stack each latent representation in the mini-batch into a matrix
$\mathbf{Z}^{t}\in\mathbb{R}^{N\times d_z}$, whose $i$-th row is $\mathbf{z}_{i}^{t}$.
The overall contrastive consistency objective $d^c(\cdot,\cdot)$ is defined as
\begin{equation}\small
d^c(\mathbf{Z}^{t_1}, \mathbf{Z}^{t_2})
=
\frac{1}{2}\Big(
\mathcal{L}_{\mathrm{CL}}^{t_1\to t_2}
+
\mathcal{L}_{\mathrm{CL}}^{t_2\to t_1}
\Big).
\label{eq:dc_symmetric}
\end{equation}
Replacing the self-consistency term $d^s(\cdot,\cdot)$ in Eq.~\eqref{equ_loss_function} with the contrastive consistency objective $d^c(\cdot,\cdot)$, the overall training objective is updated as:
\begin{equation}\small
\mathcal{L}
=
\mathbb{E}\Big[
\lambda_1 \Big(
d^b(\hat{\mathbf{Y}}_0^{t_1}, \mathbf{Y})
+ d^b(\hat{\mathbf{Y}}_0^{t_2}, \mathbf{Y})
\Big)
+ \lambda_2\, d^c(\mathbf{Z}^{t_1}, \mathbf{Z}^{t_2})
\Big].
\label{eq:gccm_objective}
\end{equation}
Here, $\hat{\mathbf{Y}}_0^{t}\in\mathbb{R}^{N\times d_k}$ denotes the predicted clean labels in the mini-batch, while $\mathbf{Y}\in\mathbb{R}^{N\times d_k}$ denotes the corresponding ground-truth clean labels stacked across the same mini-batch.




Recall from Lemma~\ref{lem:mse_shortcut_family} that, under the original objective in Eq.~\eqref{equ_loss_function}, the shortcut solution set $\mathcal{F}_{\mathrm{sc}}$ is sufficient to minimize the self-consistency term $d^s$, which reduces the joint objective to a standard supervised objective.
The following lemma shows that $\mathcal{F}_{\mathrm{sc}}$ is no longer sufficient to minimize the contrastive consistency objective $d^c$:

\begin{lemma}
\label{lem:cl_shortcut_not_sufficient}
Consider the contrastive consistency objective defined in Eq.~\eqref{eq:dc_symmetric}.
Under the shortcut solution set $\mathcal{F}_{\mathrm{sc}}$ defined in Eq.~\eqref{eq:shortcut_family}, i.e.,
$
\mathcal{F}_{\mathrm{sc}} := \{\theta:\; W_y=\mathbf{0}_{d_k\times d_h},\, W_t=\mathbf{0}_{d_t\times d_h}\}.
$
$\mathcal{F}_{\mathrm{sc}}$ is generally no longer sufficient to minimize the contrastive consistency objective $d^c$.
\end{lemma}

A formal proof is provided in Appendix~\ref{sec:proof_not_sufficient}.
Lemma~\ref{lem:cl_shortcut_not_sufficient} follows from the fact that, unlike the self-consistency objective that performs pairwise matching between two predictions, the contrastive consistency objective couples each noisy label in one noisy set with the other noisy set through negative pairs, and thus imposes an additional separation requirement across different instances.
As a result, the shortcut solution that ignores $(\mathbf{Y}_t,t)$ by down-weighting $W_y$ and $W_t$ in Eq.~\eqref{eq:additive_fusion} cannot in general satisfy the positive-agreement and negative-separation requirements simultaneously.
Therefore, unlike the self-consistency objective, the contrastive consistency objective discourages the denoising network from collapsing into a purely deterministic predictor that ignores $(\mathbf{Y}_t,t)$.

\subsection{Enhancing the Contrastive Consistency Objective with Feature Perturbation}
\label{subsec:feature_aug}
Although the proposed contrastive consistency objective $d^c(\cdot,\cdot)$ introduces negative pairs across different instances, making the shortcut solution set $\mathcal{F}_{\mathrm{sc}}$ no longer trivially sufficient to minimize $d^c(\cdot,\cdot)$, the denoising network may still potentially learn to reduce its reliance on $(\mathbf{Y}_t,t)$ by down-weighting $W_y$ and $W_t$ in Eq.~\eqref{eq:additive_fusion} during training.
The underlying reason is that the contrastive consistency objective contains the positive-pair agreement requirement, which the shortcut solution set naturally satisfies.
Under the shortcut solution set defined in Eq.~\eqref{eq:shortcut_family}, the fusion in Eq.~\eqref{eq:additive_fusion} reduces to $\mathbf{H}_t^{(0)}=W_x\mathbf{X}$, so the output latent representation matrix $\mathbf{Z}^{t}$ is independent of $(\mathbf{Y}_t,t)$:
\begin{equation}\small
\begin{aligned}
\mathbf{Z}^{t}
= f_\theta\big(\mathbf{Y}_t, t, \mathcal{G}\big)
= f_\theta\big(\mathbf{H}^{(0)}_t,\mathbf{A}\big) 
= \mathrm{Proj}_{c}\left(\mathrm{Enc}\left(\mathbf{H}^{(0)}_t, \mathbf{A}\right)\right) 
= \mathrm{Proj}_{c}\left(\mathrm{Enc}\left(W_x\mathbf{X}, \mathbf{A}\right)\right).
\end{aligned}
\label{eq:contrastive_mapping}
\end{equation}
Since the graph encoder receives the same input features $\mathbf{X}$ and $\mathbf{A}$ at $t_1$ and $t_2$, Eq.~\eqref{eq:contrastive_mapping} implies that the resulting latent representations are identical, i.e.,
$
\mathbf{Z}^{t_1}=\mathbf{Z}^{t_2}.
$
This equality makes positive-pair alignment trivial, since for each instance $i$ in the mini-batch we have $\mathbf{z}_i^{t_1}=\mathbf{z}_i^{t_2}$.
Substituting this into Eq.~\eqref{equ_latent_contrastive_obj} yields $\mathcal{L}_{\mathrm{CL}}^{t_1\to t_2}
=
-\frac{1}{N}
\sum_{i=1}^{N}
\log
\frac{
\exp(1/\tau)
}{
\sum_{j=1}^{N}
\exp\left(
\operatorname{sim}\left(\mathbf{z}_i^{t_1},\mathbf{z}_j^{t_2}\right)/\tau
\right)
}$.
In this form, the positive-pair agreement term $\operatorname{sim}\left(\mathbf{z}_i^{t_1},\mathbf{z}_i^{t_2}\right)$ becomes a constant, so the contrastive consistency objective is reduced to separating negative pairs across different instances only.
Crucially, this negative-pair separation requirement cannot constrain the model from moving toward the shortcut solution $\mathcal{F}_{\mathrm{sc}}$.
In this case, the fused representation in Eq.~\eqref{eq:additive_fusion} reduces to $\mathbf{H}_t^{(0)}=W_x\mathbf{X}$, so negative pairs can still be separated based on the input features, without relying on $(\mathbf{Y}_t,t)$.
Therefore, although the shortcut solution is no longer trivially sufficient to minimize the contrastive consistency objective, the optimization may still move toward it to make positive-pair alignment easier, and this tendency cannot be constrained by the negative-pair separation requirement.

To further prevent positive-pair agreement from being trivially satisfied by ignoring $(\mathbf{Y}_t,t)$, we propose feature perturbation on the input node features to break the identical conditioning input across two time steps $t_1$ and $t_2$, such that the shortcut solution set $\mathcal{F}_{\mathrm{sc}}$ can no longer trivially yield identical latent representations $\mathbf{Z}^{t_1}=\mathbf{Z}^{t_2}$.
Concretely, we treat the diffusion forward process defined in Section~\ref{sec:preliminaries} as a generic feature perturbation operator and inject noise into the input features accordingly.
Unlike the forward process for labels, which samples the noise level from the full range $\{1,\ldots,T\}$, feature perturbation only uses a small noise level controlled by a total perturbation time step $T_{\mathrm{per}}$ to avoid overly aggressive perturbations to the conditioning features.
By applying this lightweight perturbation twice independently, we obtain $\tilde{\mathbf{X}}_1$ and $\tilde{\mathbf{X}}_2$, which yield two perturbed graph views
$\tilde{\mathcal{G}}_1=(\tilde{\mathbf{X}}_1,\mathbf{A})$ and
$\tilde{\mathcal{G}}_2=(\tilde{\mathbf{X}}_2,\mathbf{A})$.
The perturbations are intentionally kept small and mainly serve to introduce non-identical conditioning features across the two views.
In this way, the shortcut solution $\mathcal{F}_{\mathrm{sc}}$ is in general no longer sufficient to guarantee the positive-pair agreement $\mathbf{Z}^{t_1}=\mathbf{Z}^{t_2}$.

\begin{lemma}
\label{lem:xaug_breaks_shortcut}
Let $t_1$ and $t_2$ be two noise levels sampled in consistency training.
Given two inputs $(\mathbf{Y}_{t_1}, t_1, \tilde{\mathcal{G}}_1)$ and $(\mathbf{Y}_{t_2}, t_2, \tilde{\mathcal{G}}_2)$, where
$\tilde{\mathcal{G}}_1=(\tilde{\mathbf{X}}_1,\mathbf{A})$ and
$\tilde{\mathcal{G}}_2=(\tilde{\mathbf{X}}_2,\mathbf{A})$
are independently perturbed inputs with $\tilde{\mathbf{X}}_1\neq \tilde{\mathbf{X}}_2$,
consider the shortcut solution set $\mathcal{F}_{\mathrm{sc}}$ defined in Eq.~\eqref{eq:shortcut_family}, i.e.,
$
\mathcal{F}_{\mathrm{sc}} := \{\theta:\; W_y=\mathbf{0}_{d_k\times d_h},\, W_t=\mathbf{0}_{d_t\times d_h}\}.
$
For any $\theta\in\mathcal{F}_{\mathrm{sc}}$, the latent representation $\mathbf{Z}^{t}$ produced by the contrastive head at time step $t$ can be written as
$
\mathbf{Z}^{t}
=
\mathrm{Proj}_{c}\left(
\mathrm{Enc}\left(
W_x\tilde{\mathbf{X}},\mathbf{A}
\right)
\right).
\label{eq:z_under_shortcut}
$
Then, $\mathcal{F}_{\mathrm{sc}}$ is in general no longer sufficient to guarantee the positive-pair agreement
$\mathbf{Z}^{t_1}=\mathbf{Z}^{t_2}$.
\end{lemma}

A formal proof is provided in Appendix~\ref{sec:xaug_breaks_shortcut}.
By Lemma~\ref{lem:xaug_breaks_shortcut}, positive-pair agreement can no longer be easily achieved by reducing the reliance on $(\mathbf{Y}_t,t)$ through down-weighting $W_y$ and $W_t$ in Eq.~\eqref{eq:additive_fusion},
thereby further reducing the tendency to converge toward the shortcut solution during training.

\section{Experiments}
\subsection{Experimental Settings}
\label{sub:experimental_settings}

\textbf{Datasets.}\quad 
We conduct extensive experiments covering different prediction tasks on seven benchmark datasets~\cite{dwivedi2022benchmarkinggraphneuralnetworks,dwivedi2023longrangegraphbenchmark}.
These datasets can be divided into three groups based on the prediction tasks:
(i) graph-level classification on image-based graph datasets (i.e., MNIST and CIFAR10);
(ii) graph-level regression on a molecular graph dataset (i.e., ZINC); and
(iii) node-level classification on long-range graph datasets (i.e., PascalVOC-SP and COCO-SP) and synthetic SBM datasets (i.e., PATTERN and CLUSTER).
All experiments follow standard dataset splits and evaluation metrics adopted in previous works.
More details of the datasets can be found in Appendix~\ref{sub:datasets_descriptions}.

\textbf{Baselines.}\quad
To examine whether our proposed GCCM can enhance deterministic prediction methods in a generative manner, we compare against existing conditional generative graph prediction frameworks, LGD~\cite{zhou2024unifyinggenerationpredictiongraphs} and PCL~\cite{li2025generative}. 
Since LGD originally used GPS~\cite{rampášek2023recipegeneralpowerfulscalable} as the network backbone, we also adopt GPS as the network backbone of both PCL and our GCCM for a fair comparison, which we denote as GPS+PCL and GPS+GCCM.
Please note that GCCM is not tied to any specific backbone and can be applied to other graph neural network architectures. 
Due to space limitations, we further adopt other network backbones to GCCM in Appendix~\ref{sub:cross_backbone}.
For reference, we also include the results of deterministic prediction models (GCN~\cite{kipf2017semisupervisedclassificationgraphconvolutional}, GIN~\cite{xu2019powerfulgraphneuralnetworks}, GAT~\cite{veličković2018graphattentionnetworks}, GatedGCN~\cite{bresson2018residualgatedgraphconvnets}, PNA~\cite{corso2020principalneighbourhoodaggregationgraph}, DGN~\cite{beaini2021directionalgraphnetworks}, GIN-AK+~\cite{zhao2022starssubgraphsupliftinggnn}. SAN~\cite{kreuzer2021rethinkinggraphtransformersspectral}, EGT~\cite{Hussain_2022}, GRIT~\cite{10.5555/3618408.3619379}, and GPS~\cite{rampášek2023recipegeneralpowerfulscalable}).

\begin{table*}[t]
\centering
\caption{Test performance on benchmarks from~\cite{dwivedi2022benchmarkinggraphneuralnetworks, dwivedi2023longrangegraphbenchmark}. 
Results for baselines (except GPS, GRIT, GPS-LGD, and GPS-PCL) are taken from ~\cite{rampášek2023recipegeneralpowerfulscalable}.}
\label{tab:benchmark_results}
\resizebox{\textwidth}{!}{
\begin{tabular}{lccccccc}
\toprule
\textbf{Model} 
& \multicolumn{1}{c}{\textbf{MNIST}}
& \multicolumn{1}{c}{\textbf{CIFAR10}}
& \multicolumn{1}{c}{\textbf{PATTERN}}
& \multicolumn{1}{c}{\textbf{CLUSTER}}
& \multicolumn{1}{c}{\textbf{ZINC}}
& \multicolumn{1}{c}{\textbf{PascalVOC-SP}}
& \multicolumn{1}{c}{\textbf{COCO-SP}}\\
\cmidrule(lr){2-8}
& \textbf{Accuracy} $\uparrow$
& \textbf{Accuracy} $\uparrow$
& \textbf{Accuracy} $\uparrow$
& \textbf{Accuracy} $\uparrow$
& \textbf{MAE} $\downarrow$ 
& \textbf{F1} $\uparrow$
& \textbf{F1} $\uparrow$\\
\midrule

GCN
& 90.705 $\pm$ 0.218 
& 55.710 $\pm$ 0.381 
& 71.892 $\pm$ 0.334 
& 68.498 $\pm$ 0.976
& 0.367 $\pm$ 0.011 
& 0.1268 $\pm$ 0.0060
& 0.0841 $\pm$ 0.0010\\

GIN 
& 96.485 $\pm$ 0.252 
& 55.255 $\pm$ 1.527 
& 85.387 $\pm$ 0.136 
& 64.716 $\pm$ 1.553
& 0.526 $\pm$ 0.051 
& \textemdash 
& \textemdash\\

GAT 
& 95.535 $\pm$ 0.205 
& 64.223 $\pm$ 0.455 
& 78.271 $\pm$ 0.186 
& 70.587 $\pm$ 0.447
& 0.384 $\pm$ 0.007 
& \textemdash 
& \textemdash\\

GatedGCN
& 97.340 $\pm$ 0.143 
& 67.312 $\pm$ 0.311 
& 85.568 $\pm$ 0.088 
& 73.840 $\pm$ 0.326
& 0.282 $\pm$ 0.015 
& 0.2873 $\pm$ 0.0219
& 0.2641 $\pm$ 0.0045\\

PNA
& 97.940 $\pm$ 0.120 
& 70.350 $\pm$ 0.630 
& \textemdash 
& \textemdash
& 0.188 $\pm$ 0.004 
& \textemdash 
& \textemdash\\

DGN
& \textemdash 
& 72.838 $\pm$ 0.417
& 86.680 $\pm$ 0.034 
& \textemdash
& 0.168 $\pm$ 0.003 
& \textemdash 
& \textemdash\\




GIN-AK+
& \textemdash 
& 72.190 $\pm$ 0.130 
& 86.850 $\pm$ 0.057
& \textemdash
& 0.080 $\pm$ 0.001 
& \textemdash 
& \textemdash\\

SAN
& \textemdash 
& \textemdash 
& 86.581 $\pm$ 0.037 
& 76.691 $\pm$ 0.650
& 0.139 $\pm$ 0.006 
& 0.3230 $\pm$ 0.0039
& 0.2592 $\pm$ 0.0158 \\



EGT
& \textbf{98.173 $\pm$ 0.087}
& 68.702 $\pm$ 0.409 
& 86.821 $\pm$ 0.020
& 79.232 $\pm$ 0.348
& 0.108 $\pm$ 0.009 
& \textemdash 
& \textemdash\\


GRIT
& 98.108 $\pm$ 0.111 
& \textbf{76.468 $\pm$ 0.881}
& \textbf{87.196 $\pm$ 0.076}
& \textbf{80.026 $\pm$ 0.277} 
& \textbf{0.059 $\pm$ 0.002} 
& \textemdash 
& \textemdash\\

GPS
& 98.082 $\pm$ 0.114
& 72.356 $\pm$ 0.323
& 86.648 $\pm$ 0.065 
& 77.780 $\pm$ 0.236 
& 0.072 $\pm$ 0.004
& \textbf{0.3793 $\pm$ 0.0171}
& \textbf{0.3453 $\pm$ 0.0056}\\

\midrule

GPS-LGD
& 95.461 $\pm$ 0.613
& 70.163 $\pm$ 0.357
& 85.662 $\pm$ 0.181 
& 76.773 $\pm$ 0.301
& 0.085 $\pm$ 0.001
& 0.3912 $\pm$ 0.0207
& \textemdash\\

GPS-PCL
& 98.114 $\pm$ 0.183
& 71.830 $\pm$ 0.218 
& 86.695 $\pm$ 0.073 
& 78.171 $\pm$ 0.231
& \textbf{0.069 $\pm$ 0.001}
& 0.4196 $\pm$ 0.0086
& 0.3816 $\pm$ 0.0058\\

GPS-GCCM (ours)
& \textbf{98.236 $\pm$ 0.060}
& \textbf{72.502 $\pm$ 0.287}
& \textbf{86.772 $\pm$ 0.042}
& \textbf{78.820 $\pm$ 0.187}
& 0.071 $\pm$ 0.001
& \textbf{0.4219 $\pm$ 0.0102}
& \textbf{0.3856 $\pm$ 0.0021}\\

\bottomrule
\end{tabular}
}
\vspace{-0cm}
\end{table*}

\vspace{-0.2cm}
\subsection{Main Results}
\label{sec:main_results}
We evaluate GCCM on the seven datasets introduced in Section~\ref{sub:experimental_settings}.
For all datasets, we report the mean and standard deviation over five runs with different random seeds.
The results are summarized in Table~\ref{tab:benchmark_results}.
It is clear to see that integrating GCCM into the deterministic base model GPS substantially boosts its performance across most benchmarks.
And compared with previous conditional generative approaches for graph prediction (i.e., LGD and PCL), GCCM consistently achieves superior performance on most datasets. In particular, relative to the previous vanilla consistency model for graph prediction (i.e., PCL), by effectively mitigating the shortcut solution in consistency training, GCCM leads to substantial performance improvements. 
Overall, these results demonstrate that GCCM can serve as an effective generative enhancement to deterministic backbones.

\subsection{Ablation Studies}
\label{subsec:ablation_study}
Beyond demonstrating the overall performance of GCCM, we conduct detailed ablation studies in this section to further verify the effectiveness of each proposed component of our GCCM.

\noindent\textbf{Formulating graph prediction as conditional generation.}\quad
To evaluate the benefits of formulating graph prediction as conditional generation, we introduce a variant that enhances the deterministic model GPS using a generative diffusion formulation, denoted GPS+Diffusion, which follows the design described in Section~\ref{sub:prediction_as_generation}.
To obtain stable predictions at inference time, GPS+Diffusion aggregates predictions from 10 independent noise initializations, each generated via 1000 reverse diffusion steps.
As shown in Table~\ref{tab:ab_results}, GPS+Diffusion achieves better performance on MNIST, PATTERN, and CLUSTER compared to the deterministic GPS baseline.
These improvements highlight the potential of applying conditional generative models to graph prediction tasks.

\begin{table*}[t]
\centering
\caption{Ablation studies on each component of GCCM on five benchmarks from~\cite{dwivedi2022benchmarkinggraphneuralnetworks}.
}
\vspace{-0.2cm}
\label{tab:ab_results}
\resizebox{\textwidth}{!}{
\begin{tabular}{lccccc}
\toprule
\textbf{Model} 
& \multicolumn{1}{c}{\textbf{MNIST}}
& \multicolumn{1}{c}{\textbf{CIFAR10}}
& \multicolumn{1}{c}{\textbf{PATTERN}}
& \multicolumn{1}{c}{\textbf{CLUSTER}}
& \multicolumn{1}{c}{\textbf{ZINC}}\\
\cmidrule(lr){2-6}
& \textbf{Accuracy} $\uparrow$
& \textbf{Accuracy} $\uparrow$
& \textbf{Accuracy} $\uparrow$
& \textbf{Accuracy} $\uparrow$
& \textbf{MAE} $\downarrow$\\
\midrule

GPS
& 98.082 $\pm$ 0.114
& 72.356 $\pm$ 0.323
& 86.648 $\pm$ 0.065 
& 77.780 $\pm$ 0.236 
& 0.072 $\pm$ 0.004\\

GPS+Diffusion
& 98.081 $\pm$ 0.241 
& 71.931 $\pm$ 0.347  
& 86.741 $\pm$ 0.158 
& 78.312 $\pm$ 0.251 
& 0.073 $\pm$ 0.007 \\

GPS+Consistency
& 98.114 $\pm$ 0.183
& 71.830 $\pm$ 0.218 
& 86.695 $\pm$ 0.073 
& 78.171 $\pm$ 0.231 
& 0.069 $\pm$ 0.001\\

GPS+Consistency+Contrastive
& 98.118 $\pm$ 0.056 
& 71.976 $\pm$ 0.321 
& 86.676 $\pm$ 0.071 
& 78.232 $\pm$ 0.159
& \textbf{0.068 $\pm$ 0.003}\\


GPS+GCCM (ours)
& \textbf{98.236 $\pm$ 0.060}
& \textbf{72.502 $\pm$ 0.287}
&\textbf{ 86.772 $\pm$ 0.042}
& \textbf{78.820 $\pm$ 0.187}
& 0.071 $\pm$ 0.001\\

\bottomrule
\end{tabular}
}
\vspace{-0.4cm}
\end{table*}

\noindent\textbf{Incorporating consistency training.}\quad
Building on the variant of GPS+Diffusion, we further incorporate vanilla consistency training to enable stable sampling and one-step inference, yielding a consistency-based variant denoted as GPS+Consistency.
As shown in Table~\ref{tab:ab_results}, GPS+Consistency yields more stable performance compared to GPS+Diffusion in terms of the standard deviation. 
However, it leads to reduced predictive performance, this unsatisfactory performance is mainly due to the fact that vanilla consistency training leads the model toward the shortcut solution and collapses it into a deterministic predictor.
To better illustrate this, we visualize the relative contribution of $W_y\mathbf{Y}_t + W_t\mathbf{t}_{\mathrm{emb}}(t)$ to $\mathbf{H}^{(0)}_t$ in Eq.~\eqref{eq:additive_fusion} on two representative datasets, MNIST and CLUSTER.
As shown in Fig.~\ref{fig:visualization_ablation_studies}, the heatmaps corresponding to case (a) and (d) on both datasets are dominated by dark purple colors, indicating a consistently low contribution from the terms associated with $(\mathbf{Y}_t,t)$ by down-weighting $W_y$ and $W_t$ in Eq.~\eqref{eq:additive_fusion}.
These observations are consistent with our analysis in Section~\ref{sub:prediction_as_generation}: the model might satisfy the self-consistency constraint by reducing its reliance on the noisy labels $\mathbf{Y}_t$ and the corresponding timestep $t$ (i.e., the shortcut), and thus collapse toward a deterministic predictor whose performance is close to GPS.

\begin{figure}[tb]
    \centering
    \resizebox{0.55\textwidth}{!}{
    \begin{minipage}{\textwidth}
        \begin{subfigure}[b]{0.32\linewidth}
            \centering
            \includegraphics[width=\linewidth]{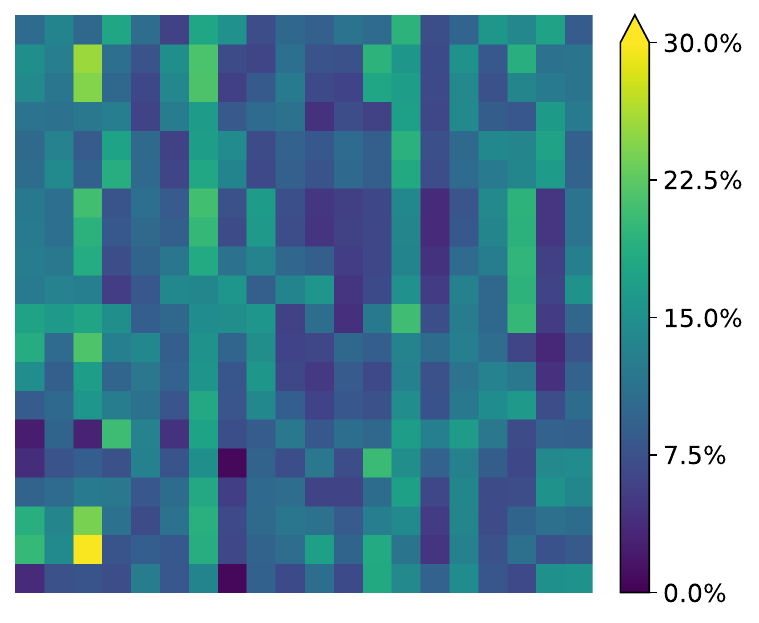}
            \caption{MNIST: GPS+Cons.}
        \end{subfigure}
        \hfill
        \begin{subfigure}[b]{0.32\linewidth}
            \centering
            \includegraphics[width=\linewidth]{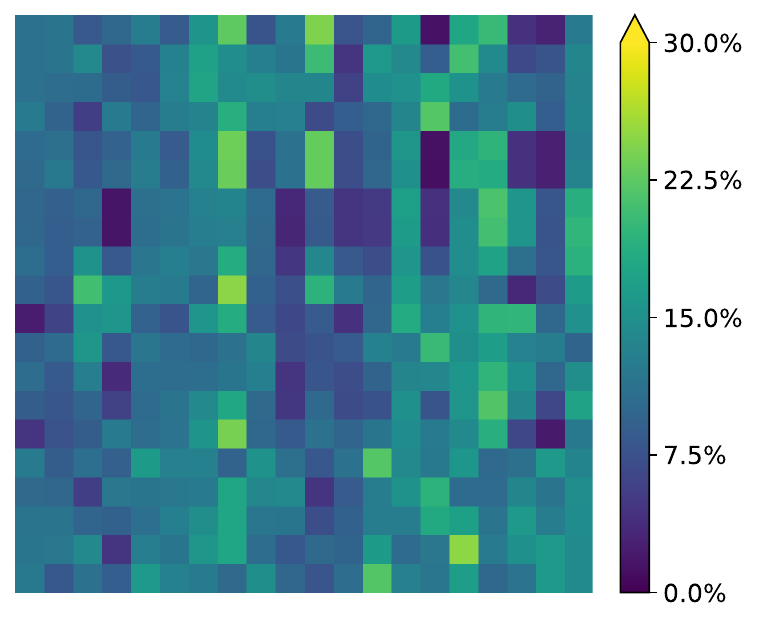}
            \caption{MNIST: GPS+Cons.+Contr.}
        \end{subfigure}
        \hfill
        \begin{subfigure}[b]{0.32\linewidth}
            \centering
            \includegraphics[width=\linewidth]{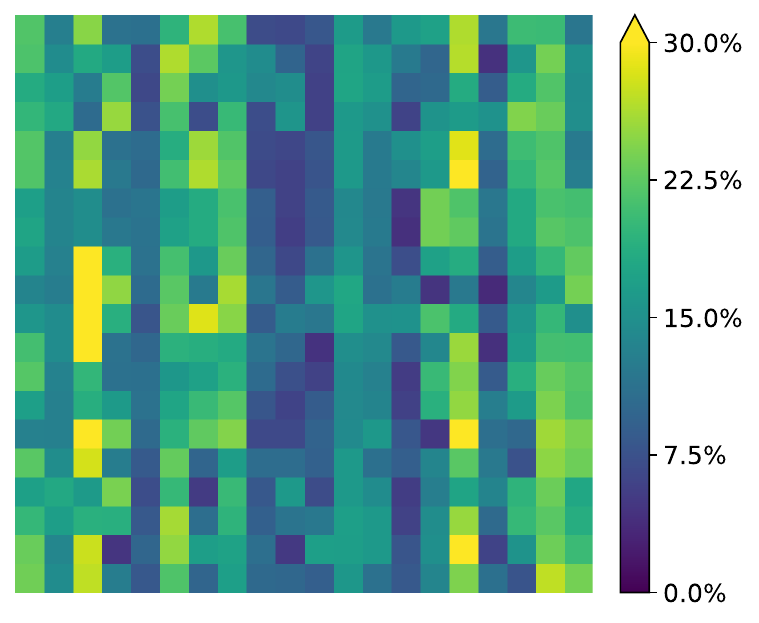}
            \caption{MNIST: GPS+GCCM}
        \end{subfigure}

        \vspace{2mm}

        \begin{subfigure}[b]{0.32\linewidth}
            \centering
            \includegraphics[width=\linewidth]{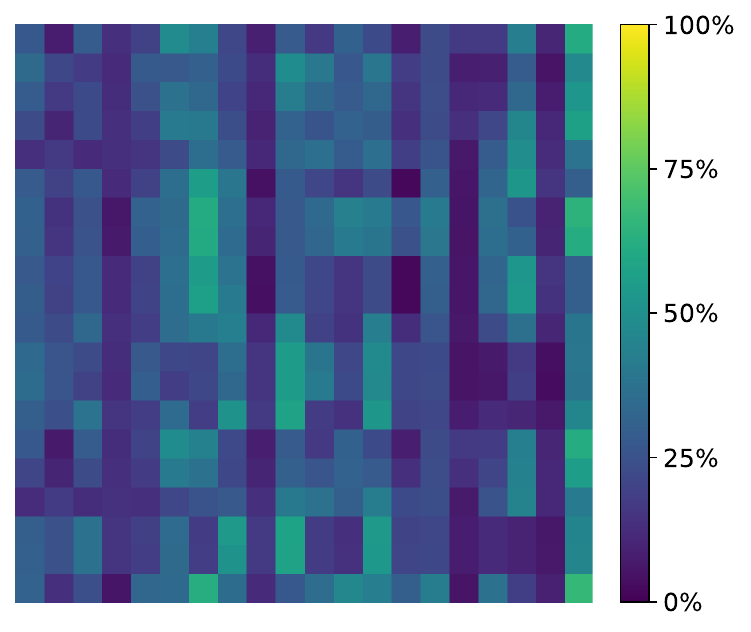}
            \caption{CLUSTER: GPS+Cons.}
        \end{subfigure}
        \hfill
        \begin{subfigure}[b]{0.32\linewidth}
            \centering
            \includegraphics[width=\linewidth]{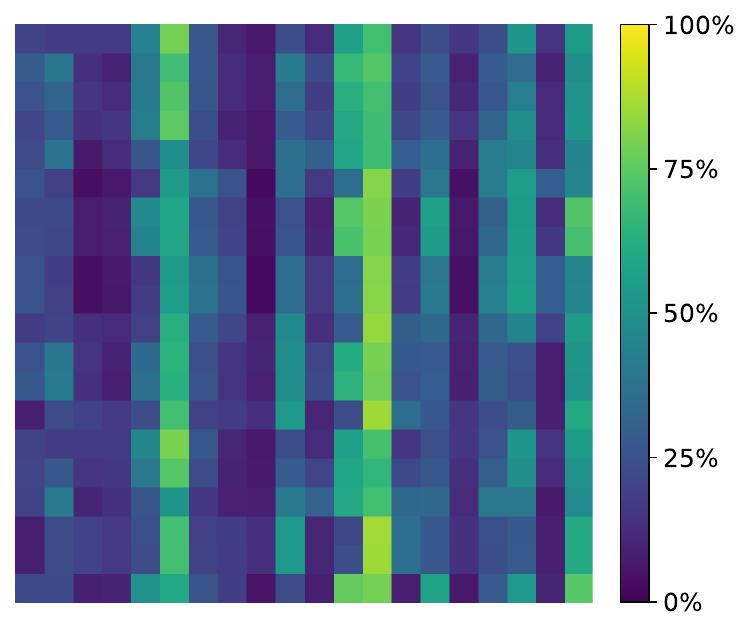}
            \caption{CLUSTER: GPS+Cons.+Contr.}
        \end{subfigure}
        \hfill
        \begin{subfigure}[b]{0.32\linewidth}
            \centering
            \includegraphics[width=\linewidth]{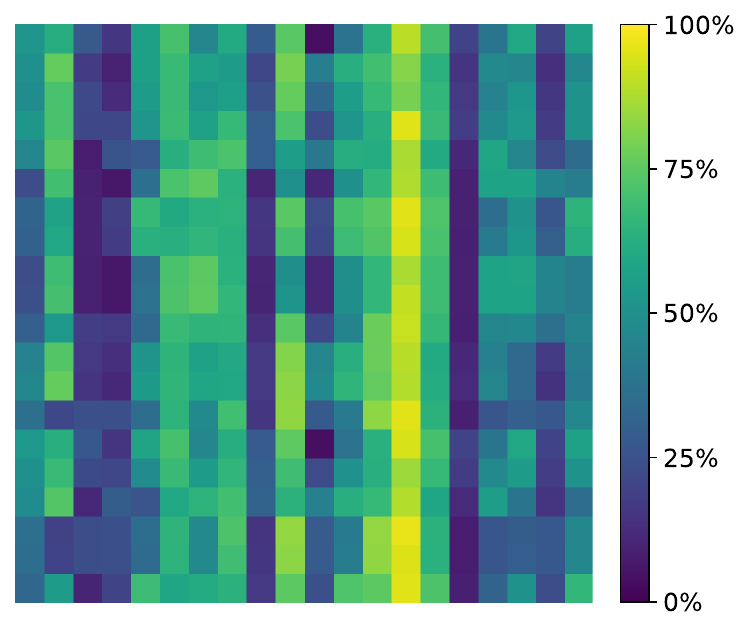}
            \caption{CLUSTER: GPS+GCCM}
        \end{subfigure}
    \end{minipage}
    }

    \vspace{-0.2cm}
    \caption{
     Heatmap visualizations of the contribution terms of $W_y\mathbf{Y}_t$ and $W_t \mathbf{t}_{\mathrm{emb}}(t)$ to $\mathbf{H}_t^{(0)}$ on \textsc{MNIST} (top) and \textsc{CLUSTER} (bottom), computed as
     $\frac{\lvert W_y \mathbf{Y}_t\rvert + \lvert W_t \mathbf{t}_{\mathrm{emb}}(t)\rvert}
     {\lvert W_y \mathbf{Y}_t\rvert + \lvert W_t \mathbf{t}_{\mathrm{emb}}(t)\rvert + \lvert W_x \mathbf{X}\rvert}$.
     The x-axis corresponds to node indices and the y-axis corresponds to hidden dimensions.
    }
    \label{fig:visualization_ablation_studies}
    \vspace{-0.5cm}
\end{figure}

\noindent\textbf{Mitigating the shortcut solution via the contrastive consistency objective.}\quad
To evaluate the effectiveness of the proposed contrastive consistency objective in mitigating the shortcut solution, we replace the self-consistency objective in GPS+Consistency with our contrastive consistency objective, yielding the variant GPS+Consistency+Contrastive.
From Table~\ref{tab:ab_results}, it is clear to see that GPS+Consistency+Contrastive demonstrates improvements over GPS+Consistency, but only to a limited extent.
Specifically, as illustrated in Fig.~\ref{fig:visualization_ablation_studies}, the corresponding heatmap (e) on CLUSTER shows only a mild mitigation effect, while the heatmap (b) on MNIST remains largely unchanged, this observation suggests that the contribution of $(\mathbf{Y}_t,t)$ is still limited.
Overall, these results indicate that although the contrastive consistency objective has the potential to alleviate the shortcut tendency toward a deterministic predictor, itself alone may not be sufficient to fully prevent the denoising network from down-weighting $W_y$ and $W_t$ in Eq.~\eqref{eq:additive_fusion} to reduce its reliance on $(\mathbf{Y}_t,t)$ to satisfy positive-pair agreement discussed in Section~\ref{subsec:feature_aug}, thereby necessitating additional training strategies.

\noindent\textbf{Enhancing the contrastive consistency objective with feature perturbation.}\quad
To further mitigate the shortcut solution in GPS+Consistency+Contrastive, we incorporate feature perturbation to break the identical feature inputs across different time steps, yielding the full model denoted as GPS+GCCM.
As shown in Table~\ref{tab:ab_results}, compared with GPS+Consistency+Contrastive, GPS+GCCM improves performance across all evaluated datasets.
Moreover, as illustrated in Fig.~\ref{fig:visualization_ablation_studies}, the corresponding heatmaps (c) and (f) on MNIST and CLUSTER become noticeably brighter and more evenly distributed, indicating that the denoising network balances its reliance on $(\mathbf{Y}_t,t)$ and $\mathbf{X}$, rather than satisfying the consistency constraint by down-weighting $W_y$ and $W_t$.
These results demonstrate that GCCM provides an effective generative enhancement for graph prediction by mitigating the shortcut solution in vanilla consistency training and avoiding collapse toward a deterministic predictor, thereby enabling the model to leverage the generative formulation to deliver consistent performance gains.
\vspace{-0.2cm}

\section{Conclusion}
In this paper, we show that consistency training for diffusion-based graph prediction can admit a shortcut solution: under identical conditioning, the denoising network may ignore noisy labels and collapse into a deterministic predictor.
To address this, we propose a Graph Contrastive Consistency Model (GCCM) with two training strategies: (1) a contrastive consistency objective with mini-batch negatives to enforce separation between difference instances, and (2) feature perturbation to break identical conditioning on input node/edge features across different noise levels.
Extensive experiments show that GCCM can effectively mitigates the short cut solution, and consistently outperforms existing diffusion-based graph prediction baselines across a broad range of datasets and prediction tasks.

\bibliographystyle{unsrtnat}
\bibliography{reference}

\newpage
\appendix

\section{Extended Related Works}
\label{appendix:related_works}
\noindent\textbf{Deterministic models for graph prediction.}\quad
Deterministic models for graph prediction typically learn a direct mapping from input graph data to target values in a supervised manner via Graph Neural Networks, where model parameters are optimized by minimizing a task-specific loss between predictions and ground-truth.
Graph neural networks perform prediction by iteratively aggregating information from local neighborhoods through message-passing mechanisms, and have given rise to a wide range of popular architectures, including GCN~\cite{kipf2017semisupervisedclassificationgraphconvolutional}, GIN~\cite{xu2019powerfulgraphneuralnetworks}, GatedGCN~\cite{bresson2018residualgatedgraphconvnets}, and PNA~\cite{corso2020principalneighbourhoodaggregationgraph}.
Later studies further improve the expressive power of GNNs by moving beyond standard message passing, leading to models such as CIN~\cite{bodnar2022weisfeilerlehmancellularcw}, CRaWI~\cite{tönshoff2023walkingweisfeilerlemanhierarchy}, and GIN-AK+~\cite{zhao2022starssubgraphsupliftinggnn}.

More recently, motivated by the strong global modeling capability of Transformers, an increasing number of works have extended Transformer-based architectures to graph prediction.
Spectral Attention Networks (SAN)~\cite{kreuzer2021rethinkinggraphtransformersspectral} employ a dual-attention design that combines global attention with edge-restricted attention, while incorporating Laplacian positional encodings to inject spectral structural information into node representations.
Graphormer~\cite{ying2021transformersreallyperformbad} adopts global attention and augments it with structural biases such as centrality and spatial encodings.
More recently, GraphGPS~\cite{rampášek2023recipegeneralpowerfulscalable} provides a general framework that integrates positional and structural encodings alongside local message passing and global attention, resulting in competitive performance across a wide range of prediction tasks on graphs.

\noindent\textbf{Generative diffusion models for graph prediction.}\quad
In recent years, diffusion models~\cite{ho2020denoisingdiffusionprobabilisticmodels,song2022denoisingdiffusionimplicitmodels,nichol2021improveddenoisingdiffusionprobabilistic,rombach2022highresolutionimagesynthesislatent} have emerged as a powerful framework for generative modeling.
They define a forward process that gradually perturbs clean data with noise, and train neural networks to reverse this process by progressively denoising.
At inference time, new samples are generated by iteratively transforming noise drawn from a predefined prior into data samples.
Motivated by their strong generative capacity, a growing body of work has extended diffusion models to graph-related generation tasks, adapting the diffusion process to structured graph data.
Representative applications include graph edit distance (GED) computation~\cite{huang2025diffgedcomputinggraphedit,huang2025unsupervisedtrainingmatchingbasedgraph}, combinatorial optimization~\cite{sun2023difuscographbaseddiffusionsolvers}, and graph generation~\cite{vignac2023digressdiscretedenoisingdiffusion,zhou2024unifyinggenerationpredictiongraphs}.
Beyond graph generation tasks, LGD~\cite{zhou2024unifyinggenerationpredictiongraphs} explores diffusion models for graph prediction by applying diffusion in a latent space to generate latent representations and then mapping them to the target space using a learned decoder.
Although effective, it still incurs substantial computational overhead due to iterative denoising at inference time.
Building upon diffusion models, consistency models~\cite{song2023consistencymodels,song2024improved} eliminate iterative reverse diffusion by directly learning a mapping from noisy samples to clean data through the self-consistency objective, enabling efficient one-step or few-step generation.
Inspired by the consistency mapping principle, PCL~\cite{li2025generative} further develops a predictive consistency framework for prediction tasks by enforcing consistent predictions across different noise levels.
This formulation enables stable sampling and one-step inference in diffusion-based graph prediction models.

\section{Further Details of the Proposed GCCM}

\subsection{Training Strategy}

\begin{algorithm}[htbp]
\caption{Training Process of GCCM}
\label{alg_GCCM_training_simple}
\begin{algorithmic}[1]
\STATE {\bfseries Input:} Dataset $\mathcal{D}$; forward process $q(\cdot)$; feature perturbation $\mathrm{Perturb}(\cdot)$;
losses $d^b(\cdot,\cdot)$ and $d^c(\cdot,\cdot)$; weights $\lambda_1,\lambda_2$
\STATE Initialize model parameters $\theta$ and optimizer
\WHILE{not converged}
\STATE Sample $(\mathcal{G} = (\mathbf{X},\mathbf{A}), \mathbf{Y}) \sim \mathcal{D}$
\STATE Sample $t_1 \sim \mathcal{U}\{1,\dots,T\}$ and set $t_2 \leftarrow \max(1,\lfloor 0.2 t_1 \rfloor)$
\STATE Sample $\mathbf{Y}_{t_1} \sim q(\mathbf{Y}_{t_1}\mid \mathbf{Y})$ and $\mathbf{Y}_{t_2} \sim q(\mathbf{Y}_{t_2}\mid \mathbf{Y})$
\STATE Sample $\tilde{\mathbf{X}}_1 \sim \mathrm{Perturb}(\mathbf{X})$ and $\tilde{\mathbf{X}}_2 \sim \mathrm{Perturb}(\mathbf{X})$
\STATE Set $\tilde{\mathcal{G}}_1 \leftarrow (\tilde{\mathbf{X}}_1,\mathbf{A})$ and $\tilde{\mathcal{G}}_2 \leftarrow (\tilde{\mathbf{X}}_2,\mathbf{A})$
\STATE $\hat{\mathbf{Y}}_0^{t_1}, \mathbf{Z}^{t_1} \leftarrow f_\theta(\mathbf{Y}_{t_1}, t_1, \tilde{\mathcal{G}}_1)$
\STATE $\hat{\mathbf{Y}}_0^{t_2}, \mathbf{Z}^{t_2} \leftarrow f_\theta(\mathbf{Y}_{t_2}, t_2, \tilde{\mathcal{G}}_2)$
\STATE $\mathcal{L} \leftarrow 
\lambda_1 \Big(
d^b(\hat{\mathbf{Y}}_0^{t_1}, \mathbf{Y})
+ d^b(\hat{\mathbf{Y}}_0^{t_2}, \mathbf{Y})
\Big)
+ \lambda_2\, d^c(\mathbf{Z}^{t_1}, \mathbf{Z}^{t_2})$
\STATE optimizer.zero\_grad()
\STATE $\mathcal{L}.\mathrm{backward}()$
\STATE optimizer.step()
\ENDWHILE
\end{algorithmic}
\end{algorithm}

The detailed training process of GCCM is presented in Algorithm~\ref{alg_GCCM_training_simple}.

\subsection{One-step Prediction}

\begin{algorithm}[htbp]
\caption{One-step Prediction in GCCM}
\label{alg_GCCM_inference}
\begin{algorithmic}[1]
\STATE {\bfseries Input:} Trained denoising network $f_\theta(\cdot)$; input graph $\mathcal{G}$
\STATE {\bfseries Output:} Predicted labels $\hat{\mathbf{Y}}_0$
\STATE Sample initial noise $\mathbf{Y}_{T}$
\STATE $\hat{\mathbf{Y}}_0 \leftarrow f_\theta(\mathbf{Y}_T, T, \mathcal{G})$
\end{algorithmic}
\end{algorithm}

The one-step prediction procedure of GCCM is summarized in Algorithm~\ref{alg_GCCM_inference}.

\section{Proofs}
\label{sec:proofs}
\subsection{Proof of Lemma~\ref{lem:mse_shortcut_family}}
\label{sec:proof_mse_shortcut}

Recall Lemma~\ref{lem:mse_shortcut_family}:

Consider the training objective
\begin{equation}
\mathcal{L}(\theta) = \mathbb{E}\left[\lambda_1\left(d^b(\hat{\mathbf{Y}}_0^{t_1}, \mathbf{Y}) + d^b(\hat{\mathbf{Y}}_0^{t_2}, \mathbf{Y})\right) + \lambda_2\, d^s(\hat{\mathbf{Y}}_0^{t_1}, \hat{\mathbf{Y}}_0^{t_2})\right],
\end{equation}
where $d^b(\cdot, \cdot)$ denotes a task-specific supervised loss and $d^s(\hat{\mathbf{Y}}_0^{t_1},\hat{\mathbf{Y}}_0^{t_2}) := \|\hat{\mathbf{Y}}_0^{t_1}-\hat{\mathbf{Y}}_0^{t_2}\|_2^2$, and each predicted label $\hat{\mathbf{Y}}_0^{t}$ can be rewritten as
\begin{equation}
\begin{aligned}
\hat{\mathbf{Y}}_0^{t}
&= f_\theta\big(\mathbf{Y}_t, t, \mathcal{G}\big)
 = f_\theta\big(\mathbf{H}^{(0)}_t, \mathbf{A}\big) \\
&= f_\theta\Big(
W_y\,\mathbf{Y}_t
+ W_t\,\mathbf{t}_{\mathrm{emb}}(t)
+ W_x\,\mathbf{X},
\mathbf{A}
\Big).
\end{aligned}
\end{equation}
Minimizing the above objective could lead the denoising network $f_\theta$ to a shortcut solution set
\begin{equation}
\mathcal{F}_{\mathrm{sc}} := \{\theta:\; W_y=\mathbf{0}_{d_k\times d_h},\, W_t=\mathbf{0}_{d_t\times d_h}\},
\end{equation}
where $d_k$ denotes the dimension of the input label, such that the noisy labels $\mathbf{Y}_t$ and timestep $t$ are ignored.

\begin{proof}
The training objective is
\begin{equation}
\mathcal{L}(\theta) = \mathbb{E}\left[\lambda_1\left(d^b(\hat{\mathbf{Y}}_0^{t_1}, \mathbf{Y}) + d^b(\hat{\mathbf{Y}}_0^{t_2}, \mathbf{Y})\right) + \lambda_2\, d^s(\hat{\mathbf{Y}}_0^{t_1}, \hat{\mathbf{Y}}_0^{t_2})\right],
\end{equation}
where $d^b(\cdot, \cdot)$ denotes a task-specific supervised loss and $d^s(\hat{\mathbf{Y}}_0^{t_1},\hat{\mathbf{Y}}_0^{t_2}) := \|\hat{\mathbf{Y}}_0^{t_1}-\hat{\mathbf{Y}}_0^{t_2}\|_2^2$.
Under the additive fusion in Eq.~\eqref{eq:additive_fusion}, the predictions can be expressed as
\begin{equation}
\hat{\mathbf{Y}}_0^{t}
\;=\;
f_\theta\big(
W_y\,\mathbf{Y}_t
+ W_t\,\mathbf{t}_{\mathrm{emb}}(t)
+ W_x\,\mathbf{X},
\mathbf{A}
\big).
\end{equation}
Accordingly, the training objective can be written as
\begin{equation}
\begin{aligned}
\mathcal{L}(\theta) = \mathbb{E}\Big[
&\lambda_1\Big(
d^b\big(f_\theta(W_y\,\mathbf{Y}_{t_1} + W_t\,\mathbf{t}_{\mathrm{emb}}(t_1) + W_x\,\mathbf{X}, \mathbf{A}),\, \mathbf{Y}\big) \\
+\; &d^b\big(f_\theta(W_y\,\mathbf{Y}_{t_2} + W_t\,\mathbf{t}_{\mathrm{emb}}(t_2) + W_x\,\mathbf{X}, \mathbf{A}),\, \mathbf{Y}\big)
\Big) \\
+\; &\lambda_2\,
\Big\|
f_\theta\big(W_y\,\mathbf{Y}_{t_1} + W_t\,\mathbf{t}_{\mathrm{emb}}(t_1) + W_x\,\mathbf{X}, \mathbf{A}\big) \\
-\; &f_\theta\big(W_y\,\mathbf{Y}_{t_2} + W_t\,\mathbf{t}_{\mathrm{emb}}(t_2) + W_x\,\mathbf{X}, \mathbf{A}\big)
\Big\|_2^2
\Big].
\end{aligned}
\end{equation}
Minimizing the training objective requires both terms to be minimized jointly. 

The boundary condition objective is minimized when
\begin{equation}
\begin{aligned}
f_\theta\big(W_y\,\mathbf{Y}_{t_1} + W_t\,\mathbf{t}_{\mathrm{emb}}(t_1) + W_x\,\mathbf{X}, \mathbf{A}\big) &= \mathbf{Y}, \\
f_\theta\big(W_y\,\mathbf{Y}_{t_2} + W_t\,\mathbf{t}_{\mathrm{emb}}(t_2) + W_x\,\mathbf{X}, \mathbf{A}\big) &= \mathbf{Y}.
\end{aligned}
\label{eq:boundary_min}
\end{equation}

The self-consistency objective is minimized when
\begin{equation}
f_\theta\big(W_y\,\mathbf{Y}_{t_1} + W_t\,\mathbf{t}_{\mathrm{emb}}(t_1) + W_x\,\mathbf{X}, \mathbf{A}\big)
=
f_\theta\big(W_y\,\mathbf{Y}_{t_2} + W_t\,\mathbf{t}_{\mathrm{emb}}(t_2) + W_x\,\mathbf{X}, \mathbf{A}\big).
\label{eq:consistency_min}
\end{equation}
Combining Eq.~\eqref{eq:boundary_min} and Eq.~\eqref{eq:consistency_min}, the joint minimization requires
\begin{equation}
f_\theta\big(W_y\,\mathbf{Y}_{t_1} + W_t\,\mathbf{t}_{\mathrm{emb}}(t_1) + W_x\,\mathbf{X}, \mathbf{A}\big)
=
f_\theta\big(W_y\,\mathbf{Y}_{t_2} + W_t\,\mathbf{t}_{\mathrm{emb}}(t_2) + W_x\,\mathbf{X}, \mathbf{A}\big)
= \mathbf{Y}.
\label{eq:joint_min}
\end{equation}

This requires $f_\theta$ to map different inputs to the same output $\mathbf{Y}$ for arbitrary $(\mathbf{Y}_{t_1}, t_1)$ and $(\mathbf{Y}_{t_2}, t_2)$. Since the two predictions share the same $\mathbf{A}$, a sufficient way to satisfy Eq.~\eqref{eq:joint_min} is to make the input to $f_\theta$ independent of $(\mathbf{Y}_t, t)$ by enforcing
\begin{equation}
W_y\,\mathbf{Y}_{t_1} + W_t\,\mathbf{t}_{\mathrm{emb}}(t_1) + W_x\,\mathbf{X}
=
W_y\,\mathbf{Y}_{t_2} + W_t\,\mathbf{t}_{\mathrm{emb}}(t_2) + W_x\,\mathbf{X}.
\label{eq:input_equal}
\end{equation}
In general, $\mathbf{Y}_{t_1}\neq \mathbf{Y}_{t_2}$ and $\mathbf{t}_{\mathrm{emb}}(t_1)\neq \mathbf{t}_{\mathrm{emb}}(t_2)$. A sufficient solution to satisfy Eq.~\eqref{eq:input_equal} for arbitrary $(\mathbf{Y}_{t_1},t_1)$ and $(\mathbf{Y}_{t_2},t_2)$ is
\begin{equation}
W_y=\mathbf{0}, \qquad W_t=\mathbf{0}.
\end{equation}
We denote the corresponding shortcut solution set $\mathcal{F}_{\mathrm{sc}} := \{\theta:\; W_y=\mathbf{0},\, W_t=\mathbf{0}\}$. Under $\mathcal{F}_{\mathrm{sc}}$, the joint minimization condition in Eq.~\eqref{eq:joint_min} reduces to $f_\theta(W_x\,\mathbf{X}, \mathbf{A}) = \mathbf{Y}$, which is exactly the supervised objective of a standard deterministic predictor. Therefore, the full training objective under $\mathcal{F}_{\mathrm{sc}}$ reduces to
\begin{equation}
\mathcal{L}(\theta) = 2\lambda_1\, \mathbb{E}\left[d^b(f_\theta(W_x\,\mathbf{X}, \mathbf{A}), \mathbf{Y})\right],
\end{equation}
and the model can minimize this objective in the same way as a deterministic predictor.
\end{proof}

\subsection{Proof of Lemma~\ref{lem:cl_shortcut_not_sufficient}}
\label{sec:proof_not_sufficient}
Recall Lemma~\ref{lem:cl_shortcut_not_sufficient}:

Consider the contrastive consistency objective defined as
\begin{equation}
d^c(\mathbf{Z}^{t_1}, \mathbf{Z}^{t_2})
=
\frac{1}{2}\Big(
\mathcal{L}_{\mathrm{CL}}^{t_1\to t_2}
+
\mathcal{L}_{\mathrm{CL}}^{t_2\to t_1}
\Big).
\end{equation}
where the $t_1\to t_2$ term of the contrastive consistency objective is defined as
\begin{equation}
\begin{aligned}
\mathcal{L}_{\mathrm{CL}}^{t_1\to t_2}
&=
-\frac{1}{N}
\sum_{i=1}^{N}
\log
\frac{
\exp\left(\operatorname{sim}\left(\mathbf{z}_i^{t_1},\mathbf{z}_i^{t_2}\right)/\tau\right)
}{
\sum_{j=1}^{N}
\exp\left(\operatorname{sim}\left(\mathbf{z}_i^{t_1},\mathbf{z}_j^{t_2}\right)/\tau\right)
}, \\
&\text{with }\;
\operatorname{sim}(\mathbf{z}_{i}^{t_1},\mathbf{z}_{j}^{t_2})
=
\frac{(\mathbf{z}_{i}^{t_1})^{T}\mathbf{z}_{j}^{t_2}}
{\|\mathbf{z}_{i}^{t_1}\|\,\|\mathbf{z}_{j}^{t_2}\|}.
\end{aligned}
\end{equation}

Under the shortcut solution set $\mathcal{F}_{\mathrm{sc}}$ defined in Eq.~\eqref{eq:shortcut_family},
$$
\mathcal{F}_{\mathrm{sc}} := \{\theta:\; W_y=\mathbf{0}_{d_k\times d_h},\, W_t=\mathbf{0}_{d_t\times d_h}\}.
$$
$\mathcal{F}_{\mathrm{sc}}$ is generally no longer sufficient to minimize the contrastive consistency objective.

\begin{proof}
\label{sec:cl_shortcut_not_sufficient}
Under the shortcut solution set defined in Eq.~\eqref{eq:shortcut_family},
\begin{equation}
\mathcal{F}_{\mathrm{sc}} := \{\theta:\; W_y=\mathbf{0},\, W_t=\mathbf{0}\},
\end{equation}
the additive fusion in Eq.~\eqref{eq:additive_fusion},
\begin{equation}
\mathbf{H}_t^{(0)} = 
W_y\,\mathbf{Y}_t
+ W_t\,\mathbf{t}_{\mathrm{emb}}(t)
+ W_x\,\mathbf{X},
\end{equation}
reduces to $\mathbf{H}_t^{(0)}=W_x\mathbf{X}$.
Therefore, the output latent representation matrix $\mathbf{Z}^{t}$ becomes independent of $(\mathbf{Y}_t,t)$:
\begin{equation}
\begin{aligned}
\mathbf{Z}^{t}
&= f_\theta\big(\mathbf{Y}_t, t, \mathcal{G}\big)
 = f_\theta\big(\mathbf{H}^{(0)}_t,\mathbf{A}\big) \\
&= \mathrm{Proj}_{c}\left(\mathrm{Enc}\left(\mathbf{H}^{(0)}_t, \mathbf{A}\right)\right) \\
&= \mathrm{Proj}_{c}\left(\mathrm{Enc}\left(W_x\mathbf{X}, \mathbf{A}\right)\right).
\end{aligned}
\label{eq:Z_independent_of_Yt_t}
\end{equation}
In particular, for two time steps $t_1$ and $t_2$ with the same input graph $\mathcal{G}=(\mathbf{X},\mathbf{A})$, we have
\begin{equation}
\mathbf{Z}^{t_1}
=
\mathrm{Proj}_{c}\left(\mathrm{Enc}\left(W_x\mathbf{X}, \mathbf{A}\right)\right)
=
\mathbf{Z}^{t_2}.
\label{eq:Z_t1_equals_Z_t2}
\end{equation}
Let $\mathbf{z}_i^{t}$ denote the $i$-th row of $\mathbf{Z}^{t}$.
Eq.~\eqref{eq:Z_t1_equals_Z_t2} implies that, for each instance $i$,
\begin{equation}
\mathbf{z}_i^{t_1} = \mathbf{z}_i^{t_2}.
\label{eq:zi_t1_equals_zi_t2}
\end{equation}

In the following, we consider the $t_1\to t_2$ term of the contrastive consistency objective; the $t_2\to t_1$ term is defined analogously.
Recall that
\begin{equation}
\mathcal{L}_{\mathrm{CL}}^{t_1\to t_2}
=
-\frac{1}{N}
\sum_{i=1}^{N}
\log
\frac{
\exp\left(\operatorname{sim}\left(\mathbf{z}_i^{t_1},\mathbf{z}_i^{t_2}\right)/\tau\right)
}{
\sum_{j=1}^{N}
\exp\left(\operatorname{sim}\left(\mathbf{z}_i^{t_1},\mathbf{z}_j^{t_2}\right)/\tau\right)
}.
\end{equation}
From Eq.~\eqref{eq:zi_t1_equals_zi_t2}, for each instance $i$ we have
\begin{equation}
\operatorname{sim}\left(\mathbf{z}_i^{t_1},\mathbf{z}_i^{t_2}\right)=1.
\end{equation}
Substituting this into the above expression yields
\begin{equation}
\mathcal{L}_{\mathrm{CL}}^{t_1\to t_2}
=
-\frac{1}{N}
\sum_{i=1}^{N}
\log
\frac{
\exp(1/\tau)
}{
\exp(1/\tau)
+
\sum_{j\neq i}
\exp\left(\operatorname{sim}\left(\mathbf{z}_i,\mathbf{z}_j\right)/\tau\right)
}.
\label{eq:cl_loss_simplified}
\end{equation}
In this form, the contribution of the positive pairs becomes constant.
Consequently, minimizing the contrastive consistency objective reduces to separating negative pairs across different instances.

We now analyze the negative-pair terms in the contrastive consistency objective.
For each anchor instance $i$, define the corresponding contrastive consistency loss in Eq.~\eqref{eq:cl_loss_simplified} as
\begin{equation}
\ell_i
=
-\log
\frac{
\exp(1/\tau)
}{
\exp(1/\tau)
+
\sum_{j\neq i}
\exp\left(\operatorname{sim}\left(\mathbf{z}_i,\mathbf{z}_j\right)/\tau\right)
}.
\label{eq:li_def}
\end{equation}
Since the numerator is fixed, minimizing $\ell_i$ requires the denominator to be as small as possible, which in turn encourages
$\operatorname{sim}(\mathbf{z}_i,\mathbf{z}_j)$ to be small for all $j\neq i$.

However, consider a particular parameter choice within $\mathcal{F}_{\mathrm{sc}}$ by additionally setting $W_x=\mathbf{0}$.
Then $\mathbf{H}^{(0)}_{t_1}=\mathbf{H}^{(0)}_{t_2}$, and the resulting latent representations collapse to the same vector for all instances, i.e.,
\begin{equation}
\mathbf{z}_1=\mathbf{z}_2=\cdots=\mathbf{z}_N.
\end{equation}
Consequently, $\operatorname{sim}(\mathbf{z}_i,\mathbf{z}_j)=1$ for all $i\neq j$, and Eq.~\eqref{eq:li_def} yields
\begin{equation}
\ell_i
=
-\log\frac{\exp(1/\tau)}{N\exp(1/\tau)}
=
\log N
>
0.
\end{equation}
Therefore, using only $W_y=\mathbf{0}$ and $W_t=\mathbf{0}$ cannot in general minimize the contrastive consistency objective, which completes the proof.
\end{proof}

\subsection{Proof of Lemma~\ref{lem:xaug_breaks_shortcut}}
\label{sec:xaug_breaks_shortcut}
Recall Lemma~~\ref{lem:xaug_breaks_shortcut}:

Let $t_1$ and $t_2$ be two noise levels sampled in consistency training.
Given two inputs $(\mathbf{Y}_{t_1}, t_1, \tilde{\mathcal{G}}_1)$ and $(\mathbf{Y}_{t_2}, t_2, \tilde{\mathcal{G}}_2)$, where
$\tilde{\mathcal{G}}_1=(\tilde{\mathbf{X}}_1,\mathbf{A})$ and
$\tilde{\mathcal{G}}_2=(\tilde{\mathbf{X}}_2,\mathbf{A})$
are independently perturbed inputs with $\tilde{\mathbf{X}}_1\neq \tilde{\mathbf{X}}_2$,
consider the shortcut solution set $\mathcal{F}_{\mathrm{sc}}$ defined in Eq.~\eqref{eq:shortcut_family},
\[
\mathcal{F}_{\mathrm{sc}} := \{\theta:\; W_y=\mathbf{0}_{d_k\times d_h},\, W_t=\mathbf{0}_{d_t\times d_h}\}.
\]
For any $\theta\in\mathcal{F}_{\mathrm{sc}}$, the latent representation $\mathbf{Z}^{t}$ produced by the contrastive head at time step $t$ can be written as
\begin{equation}
\mathbf{Z}^{t}
=
\mathrm{Proj}_{c}\left(
\mathrm{Enc}\left(
W_x\tilde{\mathbf{X}},\mathbf{A}
\right)
\right).
\end{equation}
Then, $\mathcal{F}_{\mathrm{sc}}$ is in general no longer sufficient to guarantee the positive-pair agreement
$\mathbf{Z}^{t_1}=\mathbf{Z}^{t_2}$.

\begin{proof}
For any $\theta\in\mathcal{F}_{\mathrm{sc}}$, the denoising network ignores $(\mathbf{Y}_t,t)$, and thus the latent representation produced by the contrastive head at time step $t$ can be written as
\begin{equation}
\mathbf{Z}^{t}
=
\mathrm{Proj}_{c}\left(
\mathrm{Enc}\left(
W_x\mathbf{X}, \mathbf{A}
\right)
\right).
\end{equation}
Under the two perturbed conditioning graphs
$\tilde{\mathcal{G}}_1=(\tilde{\mathbf{X}}_1, \mathbf{A})$
and
$\tilde{\mathcal{G}}_2=(\tilde{\mathbf{X}}_2, \mathbf{A})$,
the corresponding latent representations become
\begin{equation}
\mathbf{Z}^{t_1}
=
\mathrm{Proj}_{c}\left(
\mathrm{Enc}\left(
W_x\tilde{\mathbf{X}}_1, \mathbf{A}
\right)
\right),
\qquad
\mathbf{Z}^{t_2}
=
\mathrm{Proj}_{c}\left(
\mathrm{Enc}\left(
W_x\tilde{\mathbf{X}}_2, \mathbf{A}
\right)
\right).
\end{equation}
Since $\tilde{\mathbf{X}}_1 \neq \tilde{\mathbf{X}}_2$, the conditioning inputs to the network are no longer identical under $\theta\in\mathcal{F}_{\mathrm{sc}}$.
As a result, the shortcut solution set can no longer guarantee $\mathbf{Z}^{t_1}=\mathbf{Z}^{t_2}$.
\end{proof}

\section{Experimental Settings}
\label{sec:add_experimental_settings}

\subsection{Datasets Descriptions}
\label{sub:datasets_descriptions}

\begin{table*}[htbp]
\centering
\caption{Dataset statistics.}
\label{tab:datasets}
\resizebox{\textwidth}{!}{
\begin{tabular}{lrrrccc}
\toprule
Dataset & \# Graphs & Avg. \# nodes & Avg. \# edges  & Prediction level & Prediction task & Metric \\
\midrule
\textbf{MNIST}    & 70{,}000 &  70.6 &  564.5   & graph          & 10-class classif. & Accuracy \\
\textbf{CIFAR10}  & 60{,}000 & 117.6 &  941.1   & graph          & 10-class classif. & Accuracy \\
\textbf{PATTERN}  & 14{,}000 & 118.9 & 3{,}039.3  & inductive node & binary classif.   & Accuracy \\
\textbf{CLUSTER}  & 12{,}000 & 117.2 & 2{,}150.9  & inductive node & 6-class classif.  & Accuracy \\
\textbf{ZINC}     & 12{,}000 &  23.2 &   24.9    & graph          & regression        & Mean Abs.\ Error \\
\midrule
\textbf{PascalVOC-SP}  & 11{,}355 & 479.4 & 2{,}710.5  & inductive node & 21-class classif.  & F1 score \\
\textbf{COCO-SP}  & 123{,}286 & 476.9 & 2{,}693.7  & inductive node & 81-class classif.  & F1 score \\
\bottomrule
\end{tabular}
}
\end{table*}

\noindent\textbf{ZINC~\cite{dwivedi2022benchmarkinggraphneuralnetworks}.}\quad
ZINC contains 12K molecular graphs from the ZINC database.
Nodes represent heavy atoms (28 atom types) and edges represent chemical bonds (3 bond types).
The task is graph-level regression of constrained solubility (logP), evaluated by MAE.
The predefined 10K/1K/1K train/validation/test split is used.

\noindent\textbf{MNIST-SP and CIFAR10-SP~\cite{dwivedi2022benchmarkinggraphneuralnetworks}.}\quad
MNIST-SP and CIFAR10-SP are graph equivalents of the MNIST and CIFAR10 image datasets, constructed from SLIC superpixels.
Each image is converted into a graph by connecting superpixel nodes via an 8-nearest-neighbor graph.
Both are 10-class graph-level classification tasks evaluated by accuracy,
using the standard splits (MNIST: 55K/5K/10K; CIFAR10: 45K/5K/10K for train/validation/test graphs).

\noindent\textbf{PATTERN and CLUSTER~\cite{dwivedi2022benchmarkinggraphneuralnetworks}.}\quad
PATTERN and CLUSTER are synthetic datasets sampled from a Stochastic Block Model (SBM) and are designed for inductive node classification.
PATTERN requires identifying nodes that belong to one of 100 predefined subgraph patterns.
CLUSTER requires classifying nodes into 6 SBM-generated clusters (cluster IDs).
Accuracy is reported following the standard benchmark splits.

\noindent\textbf{PascalVOC-SP and COCO-SP~\cite{dwivedi2023longrangegraphbenchmark}.}\quad
PascalVOC-SP and COCO-SP are superpixel-graph datasets derived from Pascal VOC and MS COCO images via SLIC superpixelization.
Both are inductive node classification tasks that assign each superpixel node to an object category (semantic segmentation at the superpixel level),
evaluated by accuracy using the official benchmark splits.

\subsection{Implementation Details}

\begin{table*}[htbp]
\centering
\caption{GCCM hyperparameters for seven datasets from~\cite{dwivedi2022benchmarkinggraphneuralnetworks} and~\cite{dwivedi2023longrangegraphbenchmark}.}
\label{tab:gccm_hyperparameters}
\resizebox{\textwidth}{!}{
\begin{tabular}{lccccccc}
\toprule
Hyperparameter & \textbf{MNIST} & \textbf{CIFAR10} & \textbf{PATTERN}  & \textbf{CLUSTER} & \textbf{ZINC} & \textbf{PascalVOC-SP} & \textbf{COCO-SP} \\
\midrule
\# GPS Layer  & 3 & 3 & 6 & 16 & 10 & 4 & 4 \\
\# Epochs     & 100 & 100 & 100 & 100 & 2000 & 300 & 300 \\
Hidden dim $d_h$   & 52 & 52 & 64 & 48 & 64 & 96 & 96 \\
Batch size   & 16 & 16 & 32 & 16 & 32 & 32 & 32 \\
Total diffusion steps $T$  & 1000 & 1000 & 1000 & 1000 & 1000 & 1000 & 1000 \\
Time-decay ratio $\alpha$  & 0.2 & 0.2 & 0.2 & 0.2 & 0.2 & 0.2 & 0.2 \\
Boundary weight $\lambda_1$   & 1.0 & 1.0 & 1.0 & 1.0 & 1.0 & 1.0 & 1.0 \\
Consistency weight $\lambda_2$   & 0.1 & 0.1 & 0.1 & 0.1 & 0.3 & 0.1 & 0.1 \\
Latent dim $d_z$      & 52 & 52 & 64 & 48 & 64 & 96 & 96 \\
Temperature parameter $\tau$    & 1.0 & 0.5 & 0.4 & 0.5 & 0.3 & 0.3 & 0.3 \\
Perturbation type   & Continuous & Continuous & Discrete & Discrete & Discrete & Continuous & Continuous \\
Total perturbation time step $T_{per}$ & 10 & 10 & 50 & 50 & 10 & 10 & 10 \\
\bottomrule
\end{tabular}
}
\end{table*}

In this section, we summarize the GCCM-specific hyperparameters used across all datasets in Table~\ref{tab:gccm_hyperparameters}.
We adopt GPS~\cite{rampášek2023recipegeneralpowerfulscalable} as the denoising backbone; therefore, the denoising network architecture in GCCM follows the standard GPS configuration and is not repeated here.

For all datasets, we set the total number of diffusion steps to $T=1000$.
During the forward process on node/graph labels, we sample $t_1$ uniformly at random for each graph in a mini-batch from $\{1,\ldots,T\}$, and set $t_2=\lfloor \alpha t_1 \rfloor$ with a fixed time-decay ratio $\alpha=0.2$.
For consistency training, we set $\lambda_1=1.0$ and $\lambda_2=0.1$ in Eq.~\eqref{eq:gccm_objective} to balance the contributions of different objective terms.
For a fair comparison, the same $\lambda_1$ and $\lambda_2$ are used for all consistency-based ablation variants in Section~\ref{subsec:ablation_study}.
We compute the contrastive consistency objective in a latent space by applying an additional projection head $\mathrm{Proj}_c(\cdot)$, and set the latent dimension to match the hidden dimension, i.e., $d_z=d_h$.
Feature perturbation is selected according to the input feature type of each dataset.
For both discrete and continuous features, we construct perturbations using the corresponding forward noising process described in Section~\ref{sub:Diffusion_Models}.
For discrete features, we set $T_{\mathrm{per}}=50$, while for continuous features we set $T_{\mathrm{per}}=10$ to avoid overly disrupting feature semantics.



\section{Additional Experimental Results}
\label{sec:additional_results}

\begin{table*}[htbp]
\centering
\caption{Test performance on five benchmarks from~\cite{dwivedi2022benchmarkinggraphneuralnetworks}. 
Shown is the mean $\pm$ std of five runs with different seeds.}
\label{tab:different_backbone}
\resizebox{\textwidth}{!}{
\begin{tabular}{lccccc}
\toprule
\textbf{Model} 
& \textbf{MNIST}
& \textbf{CIFAR10}
& \textbf{PATTERN}
& \textbf{CLUSTER}
& \textbf{ZINC}\\
\cmidrule(lr){2-6}
& \textbf{Accuracy} $\uparrow$
& \textbf{Accuracy} $\uparrow$
& \textbf{Accuracy} $\uparrow$
& \textbf{Accuracy} $\uparrow$
& \textbf{MAE} $\downarrow$ \\
\midrule

SAN
& \textemdash 
& \textemdash 
& 86.581 $\pm$ 0.037 
& 76.691 $\pm$ 0.650
& 0.139 $\pm$ 0.006 \\

SAN-LGD
& \textemdash 
& \textemdash 
& 80.742 $\pm$ 0.315
& 68.711 $\pm$ 0.676
& 0.236 $\pm$ 0.012 \\

SAN-PCL
& \textemdash 
& \textemdash 
& 86.768 $\pm$ 0.064
& 77.461 $\pm$ 0.427
& 0.140 $\pm$ 0.002 \\

SAN-GCCM (ours)
& \textemdash 
& \textemdash 
& \textbf{86.811 $\pm$ 0.021}
& \textbf{78.481 $\pm$ 0.176}
& \textbf{0.131 $\pm$ 0.003} \\

\midrule

GPS
& 98.082 $\pm$ 0.114
& 72.356 $\pm$ 0.323
& 86.648 $\pm$ 0.065 
& 77.780 $\pm$ 0.236 
& 0.072 $\pm$ 0.004 \\

GPS-LGD
& 95.461 $\pm$ 0.613
& 70.163 $\pm$ 0.357
& 85.662 $\pm$ 0.181 
& 76.773 $\pm$ 0.301
& 0.085 $\pm$ 0.001\\

GPS-PCL
& 98.114 $\pm$ 0.183
& 71.830 $\pm$ 0.218 
& 86.695 $\pm$ 0.073 
& 78.171 $\pm$ 0.231
& \textbf{0.069 $\pm$ 0.001} \\

GPS-GCCM (ours)
& \textbf{98.236 $\pm$ 0.060}
& \textbf{72.502 $\pm$ 0.287}
& \textbf{86.772 $\pm$ 0.042}
& \textbf{78.820 $\pm$ 0.187}
& 0.071 $\pm$ 0.001 \\

\midrule

GRIT
& 98.108 $\pm$ 0.111 
& \textbf{76.468 $\pm$ 0.881}
& 87.196 $\pm$ 0.076 
& \textbf{80.026 $\pm$ 0.277} 
& 0.059 $\pm$ 0.002 \\

GRIT-LGD
& 96.544 $\pm$ 0.615
& 71.167 $\pm$ 1.311  
& 84.600 $\pm$ 1.841
& 64.913 $\pm$ 1.639
& 0.118 $\pm$ 0.011 \\

GRIT-PCL
& 97.997 $\pm$ 0.126
& 75.720 $\pm$ 0.734  
& 86.873 $\pm$ 0.060
& 78.643 $\pm$ 0.836
& 0.063 $\pm$ 0.005 \\

GRIT-GCCM (ours)
& \textbf{98.473 $\pm$ 0.095}
& 75.905 $\pm$ 0.842
& \textbf{87.256 $\pm$ 0.084}  
& 79.816 $\pm$ 0.137  
& \textbf{0.059 $\pm$ 0.004} \\

\bottomrule
\end{tabular}
}
\vspace{-0.2cm}
\end{table*}

\subsection{Generalization Across Different Backbones}
\label{sub:cross_backbone}
The main results in Section~\ref{sec:main_results} demonstrate the effectiveness of GCCM on the GPS backbone. 
A natural question is whether the proposed framework generalizes to other graph neural network architectures.
To investigate this, we apply GCCM and the recent conditional generative baselines (i.e., LGD~\cite{zhou2024unifyinggenerationpredictiongraphs} and PCL~\cite{li2025generative}) on top of two additional Graph Transformer backbones, SAN~\cite{kreuzer2021rethinkinggraphtransformersspectral} and GRIT~\cite{10.5555/3618408.3619379}.
The results are summarized in Table~\ref{tab:different_backbone}.
As shown in the table, GCCM consistently outperforms the original backbones and the baseline generative prediction frameworks LGD and PCL across all datasets under the same denoising backbone, confirming its generalizability and effectiveness.
This consistent improvement suggests that GCCM's benefit stems from mitigating the shortcut solution in consistency training itself, rather than from interactions with any particular backbone architecture, indicating that GCCM provides a general training framework applicable to diverse graph neural networks.

\begin{table}[htbp]
\centering
\caption{Efficiency comparison of GCCM against baselines across five benchmarks from~\cite{dwivedi2022benchmarkinggraphneuralnetworks}.}
\label{tab:efficiency}
\resizebox{\textwidth}{!}{
\begin{tabular}{llccccc}
\toprule
\textbf{Metric} & \textbf{Model} & \textbf{MNIST} & \textbf{CIFAR10} & \textbf{PATTERN} & \textbf{CLUSTER} & \textbf{ZINC} \\
\midrule
\multirow{4}{*}{\shortstack[l]{Time/Epoch\\(s)}}
  & GPS            & 162.22 & 129.68 & 35.63  & 99.71  & 19.33 \\
  & GPS-Diffusion  & 171.85 & 144.79 & 54.45  & 152.66 & 20.84 \\
  & GPS-PCL        & 286.99 & 248.24 & 98.70  & 277.24 & 45.70 \\
  & GPS-GCCM (ours)& 302.71 & 243.92 & 97.72  & 276.06 & 46.67 \\
\midrule
\multirow{4}{*}{\shortstack[l]{GPU Memory\\(MiB)}}
  & GPS            & 418    & 506    & 14{,}490 & 9{,}878  & 408 \\
  & GPS-Diffusion  & 462    & 532    & 14{,}640 & 10{,}788 & 418 \\
  & GPS-PCL        & 538    & 696    & 20{,}454 & 15{,}862 & 484 \\
  & GPS-GCCM (ours)& 584    & 700    & 20{,}580 & 16{,}224 & 486 \\
\midrule
\multirow{4}{*}{\shortstack[l]{Inference Time\\(s)}}
  & GPS            & 6.21   & 5.36   & 2.37     & 2.91     & 0.72 \\
  & GPS-Diffusion  & 5{,}756.25 & 5{,}087.50 & 1{,}868.58 & 2{,}398.31 & 492.65 \\
  & GPS-PCL        & 6.32   & 6.35   & 2.81     & 3.43     & 0.83 \\
  & GPS-GCCM (ours)& 6.69   & 6.41   & 3.01     & 3.28     & 0.84 \\
\bottomrule
\end{tabular}
}
\end{table}

\subsection{Runtime and Cost Analysis}
We report training cost from three aspects: per-epoch training time, GPU memory usage, and inference time.
The complete results across all five benchmarks from~\cite{dwivedi2022benchmarkinggraphneuralnetworks} are summarized in Table~\ref{tab:efficiency}.
Compared to GPS-PCL, GCCM introduces only negligible overhead, with similar training time and GPU memory usage.
The additional cost mainly comes from feature perturbation and the contrastive consistency objective.
For inference, compared to GPS-Diffusion which requires iterative denoising, GCCM is substantially faster due to one-step inference, achieving inference speed comparable to the deterministic GPS baseline while retaining the generative formulation.

\section{Discussion of Broader Impacts \& Limitations}
\label{sec:limitations}
GCCM has demonstrated superior performance on general graph prediction tasks. 
Moreover, the proposed framework is applicable to a broader applications of prediction tasks on graphs, such as combinatorial optimization~\cite{sun2023difuscographbaseddiffusionsolvers,li2025fastt2toptimizationconsistency} and graph edit distance~\cite{huang2025diffgedcomputinggraphedit,huang2025unsupervisedtrainingmatchingbasedgraph}. 
In these settings, directly applying consistency training to accelerate diffusion-based models may similarly converge to the shortcut solution and collapse into a deterministic predictor, leading to unsatisfactory performance. 
In such cases, GCCM remains applicable for mitigating the shortcut solution, leading to improved performance. 
However, due to limited GPU computational resources, we are not able to evaluate our framework on broader applications, and we will leave this to future works.



\end{document}